\documentclass{article}
\usepackage[utf8]{inputenc}
\usepackage[affil-it]{authblk}
\usepackage[margin=1in]{geometry}
\usepackage[titletoc,title]{appendix}
\usepackage[dvipsnames,table,xcdraw]{xcolor}
\usepackage{color}
\usepackage{setspace}
\usepackage{amsmath}

\usepackage{amsmath,amsfonts,amssymb,amsthm,mathtools,bigints}
\usepackage{algorithmic}
\usepackage{algorithm}

\usepackage{graphicx,float}
\usepackage{subcaption}  
\usepackage{booktabs}
\usepackage{booktabs,multirow}
\usepackage{epsfig}

\usepackage[colorlinks=true, pdfborder={0 0 0}, linkcolor=red]{hyperref}

\usepackage[style=authoryear,date=year,maxcitenames=2,maxbibnames=20,giveninits,url=false,backend=biber,uniquename=false,uniquelist=false]{biblatex}
\bibliography{main}
\usepackage{wasysym}

\allowdisplaybreaks
\title{\textbf{Probabilistic Guarantees of Stochastic Recursive Gradient in Non-Convex Finite Sum Problems}}
\author[1]{Yanjie Zhong}
\author[1]{Jiaqi Li\thanks{Corresponding author: Jiaqi Li, lijiaqi@wustl.edu}}
\author[1]{Soumendra Lahiri}

\affil[1]{Department of Statistics and Data Science, Washington University in St. Louis}

\date{\today}

\begin{document}

\maketitle

\begin{abstract}
This paper develops a new dimension-free Azuma-Hoeffding type bound on summation norm of a martingale difference sequence with random individual bounds. With this novel result, we provide high-probability bounds for the gradient norm estimator in the proposed algorithm Prob-SARAH, which is a modified version of the StochAstic Recursive grAdient algoritHm (SARAH), a state-of-art variance reduced algorithm that achieves optimal computational complexity in expectation for the finite sum problem. The in-probability complexity by Prob-SARAH matches the best in-expectation result up to logarithmic factors. Empirical experiments demonstrate the superior probabilistic performance of Prob-SARAH on real datasets compared to other popular algorithms.

\vspace{1cm}
\noindent\textit{\textbf{Keywords:}} machine learning, variance-reduced method, stochastic gradient descent, non-convex optimization

\end{abstract}

\setlength{\parindent}{15pt}

\newcommand{\theHalgorithm}{\arabic{algorithm}}

\def\bdy{\mathbf y}
\def\bdx{\mathbf x}
\def\bdX{\mathbf X}
\def\bdz{\mathbf z}
\def\bds{\mathbf s}
\def\bdw{\mathbf w}
\def\bdepsilon{\boldsymbol{\epsilon}}
\def\bdnu{\boldsymbol{\nu}}
\def\bdrho{\boldsymbol{\rho}}
\def\bdw{\mathbf w}

\def\EE{\mathbb E}
\def\NN{\mathbb N}
\def\PP{\mathbb P}
\def\RR{\mathbb R}
\def\ZZ{\mathbb Z}

\def\D{\mathcal D}

\def\FFF{\mathcal F}
\def\III{\mathcal I}
\def\JJJ{\mathcal J}
\def\AAA{\mathcal A}
\def\OOO{\mathcal O}
\def\upj{^{(j)}}
\def\upT{^{(T)}}
\def\lp({\left(}
\def\rp){\right)}

\newtheorem{theorem}{Theorem}[section]
\newtheorem{lemma}{Lemma}[section]
\newtheorem{corollary}{Corollary}[section]
\newtheorem{conclusion}{Conclusion}[section]
\newtheorem{proposition}{Proposition}[section]
\newtheorem{definition}{Definition}[section]
\newtheorem{assumption}{Assumption}[section]
\newtheorem{remark}{Remark}[section]

\newpage

\section{Introduction}

We consider the popular non-convex finite sum optimization problem in this work, that is, estimating $\bdx^* \in \D \subseteq \RR^d$ minimizing the following loss function
\begin{equation}
f(\bdx) = \frac{1}{n}\sum\limits_{i=1}\limits^{n}f_i(\bdx),\ \bdx \in \D,
\label{mainproblem}
\end{equation}
where $f_i: \RR^d\mapsto\RR$ is a potentially non-convex function on some compact set $\D$. Such non-convex problems lie at the heart of many applications of statistical learning \textcite{james2013introduction} and machine learning \textcite{goodfellow2016deep}. 

Unlike convex optimization problems, in general, non-convex problems are intractable and the best we can expect is to find a stationary point. Given a target error $\varepsilon$, since $\nabla f(\bdx^*)=0$, we aim to find an estimator $\hat{\bdx}$ such that roughly $\|\nabla f(\hat{\bdx})\|\leq \varepsilon$, where $\nabla f(\cdot)$ denotes the gradient vector the loss function $f$ and $\|\cdot\|$ is the operator norm. With a non-deterministic algorithm, the output $\hat{\bdx}$ is always stochastic, and the most frequently considered measure of error bound is in expectation, i.e.,
\begin{equation}
    \label{eq_expectation_bound}
    \EE \|\nabla f(\hat{\bdx})\|^2\leq \varepsilon^2.
\end{equation}
There has been a substantial amount of work providing upper bounds on computational complexity needed to achieve the in-expectation bound. However, in practice, we only run a stochastic algorithm for once and an in-expectation bound cannot provide a convincing bound in this situation. Instead, a high-probability bound is more appropriate by nature. Given a pair of target errors $(\varepsilon,\delta)$, we want to obtain an estimator $\hat{\bdx}$ such that with probability at least $1-\delta$, $\|\nabla f(\hat{\bdx})\|\leq \varepsilon$, that is
\begin{equation}
    \label{eq_probability_bound}
    \PP\big(\|\nabla f(\hat{\bdx})\|\leq \varepsilon\big) \ge 1-\delta.
\end{equation}
Though the Markov inequality might help, in general, an in-expectation bound cannot be simply converted to an in-probability bound with a desirable dependency on $\delta$. It would be important to prove upper bounds on high-probability complexity, which ideally should be polylogorithmic in $\delta$ and with polynomial terms comparable to the in-expectation complexity bound. 

Gradient-based methods are favored by practitioners due to simplicity and efficiency and have been widely studied by researchers in the non-convex setting (\cite{nesterov2003introductory,ghadimi2013stochastic,allen2016variance,reddi2016stochastic,fang2018spider,wang2019spiderboost}). Among numerous gradient-based methods, the StochAstic Recursive grAdient algoritHm (SARAH) (\cite{nguyen2017sarah,nguyen2017stochastic,wang2019spiderboost}) is the one with the best first-order guarantee as given an in-expectation error target, in both of convex and non-convex finite sum problems. It is worth noticing that \textcite{li2019ssrgd} attempted to show that a modified version of SARAH is able to approximate the second-order stationary point with a high probability. However, we believe that their application of the martingale Azuma-Hoeffding inequality is unjustifiable because the bounds are potentially random and uncontrollable. In this paper, we shall provide a correct dimension-free martingale Azuma-Hoeffding inequality with rigorous proofs and leverage it to show in-probability properties for SARAH-based algorithms in the non-convex setting.

\subsection{Related Works}
\begin{itemize}
    \item \textbf{High-Probability Bounds:} While most works in the literature of optimization provide in-expectation bounds, there is only a small fraction of works discussing bounds in the high probability sense. \textcite{kakade2009generalization} provide a high-probability bound on the excess risk given a bound on the regret. \textcite{jain2019making}, \textcite{harvey2019tight,harvey2019simple} derive some high-probability bounds for SGD in convex online optimization problems. \textcite{zhou2018convergence,li2020high} prove high-probability bounds for several adaptive methods, including AMSGrad, RMSProp and Delayed AdaGrad with momentum. All these works rely on (generalized) Freedman's inequality or the concentration inequality given in Lemma 6 in \textcite{jin2019short}. Different from them, our high-probability results are built on a novel Azuma-Hoeffding type inequality proved in this work and Corollary 8 from \textcite{jin2019short}. In addition, we notice that \textcite{li2019ssrgd} provide some probabilistic bounds on a SARAH-based algorithm. However, we believe their use of the plain martingale Azuma-Hoeffding inequality is not justifiable. \textcite{fang2018spider} show in-probability upper bound for SPIDER. Nevertheless, SPIDER's practical performance is inferior due to its accuracy-dependent small step size \textcite{tran2019hybrid,wang2019spiderboost}.
    \item \textbf{Variance-Reduced Methods in Non-Convex Finite Sum Problems:} Since the invention of the variance-reduction technique in \textcite{roux2012stochastic,johnson2013accelerating,defazio2014saga}, there has been a
large amount of work incorporating this efficient technique to methods targeting the non-convex finite-sum problem. Subsequent methods, including SVRG (\cite{allen2016variance,reddi2016stochastic,li2018simple}), SARAH (\cite{nguyen2017sarah,nguyen2017stochastic}), SCSG (\cite{lei2017less,lei2017non,horvath2020adaptivity}), SNVRG (\cite{zhou2018convergence}), SPIDER (\cite{fang2018spider}), SpiderBoost (\cite{wang2019spiderboost}) and PAGE (\cite{li2021page}), have greatly reduced computational complexity in non-convex problems.
\end{itemize}

\subsection{Our Contributions}
\begin{itemize}
    \item \textbf{Dimension-Free Martingale Azuma-Hoeffding inequality:} To facilitate our probabilistic analysis, we provide a novel Azuma-Hoeffding type bound on the summation norm of a martingale difference sequence. The novelty is two-fold. Firstly, same as the plain martingale Azuma-Hoeffding inequality, it provides a dimension-free bound. In a recent paper, a sub-Gaussian type bound has been developed by \textcite{jin2019short}. However, their results are not dimension-free. Our technique in the proof is built on a classic paper by \textcite{pinelis1992approach} and is completely different from the random matrix technique used in \textcite{jin2019short}. Secondly, our concentration inequality allows random bounds on each element of the martingale difference sequence, which is much tighter than a large deterministic bound. It should be highlighted that our novel concentration result perfectly suits the nature of SARAH-style methods where the increment can be characterized as a martingale difference sequence and it can be further used to analyze other algorithms beyond the current paper.
    \item \textbf{In-probability error bounds of stochastic recursive gradient:} We design a SARAH-based algorithm, named Prob-SARAH, adapted to the high-probability target and provably show its good in-probability properties. Under appropriate parameter setting, the first order complexity needed to achieve the in-probability target is $\tilde{\OOO}\left( \frac{1}{\varepsilon^3}\wedge \frac{\sqrt{n}}{\varepsilon^2}\right)$, which matches the best known in-expectation upper bound up to some logarithmic factors (\cite{zhou2018convergence,wang2019spiderboost,horvath2020adaptivity}). We would like to point out that the parameter setting used to achieve such complexity is semi-adaptive to $\varepsilon$. That is, only the final stopping rule relies on $\varepsilon$ while other key parameters are independent of $\varepsilon$, including step size, mini-batch sizes, and lengths of loops. 
    \item \textbf{Probabilistic analysis of SARAH for non-convex finite sum:} Existing literature on the bounds of SARAH is mostly focusing on the strongly convex or general convex settings. We extend the case to the non-convex scenarios, which can be considered as a complimentary study to the stochastic recursive gradient in probability.
\end{itemize}

\subsection{Notation}
For a sequence of sets $\AAA_1,\AAA_2,\ldots$, we denote the smallest sigma algebra containing $\AAA_i$, $i\ge 1$, by $\sigma\big(\bigcup_{i=1}^{\infty}\AAA_i\big)$. By abuse of notation, for a random variable $\bdX$, we denote the sigma algebra generated by $\bdX$ by $\sigma(\bdX)$. We define constant $C_e = \sum_{i=0}^{\infty}i^{-2}$. For two scalars $a,b\in\RR$, we denote $a\wedge b = \min\{a,b\}$ and $a \vee b =\max\{a,b\}$. When we say a quantity $T$ is $\OOO_{\theta_1,\theta_2}(\theta_3)$ for some $\theta_1,\theta_2,\theta_3\in \RR$, there exists a $g\in \RR$ polylogarithmic in $\theta_1$ and $\theta_2$ such that $T\leq g\cdot \theta_3$, and similarly $\tilde\OOO_{(\cdot)}(\cdot)$ is defined the same but up to a logarithm factor. 

\section{Prob-SARAH Algorithm}
The algorithm Prob-SARAH proposed in our work is a modified version of SpiderBoost (\cite{wang2019spiderboost}) and SARAH (\cite{nguyen2017sarah,nguyen2017stochastic}). Since the key update structure is originated from (\cite{nguyen2017sarah}), we call our modified algorithm Prob-SARAH. In fact, it can also be viewed as a generalization of the SPIDER algorithm introduced in (\cite{fang2018spider}).

We present the Prob-SARAH in Algorithm \ref{algo1}, and here, we provide some explanation of the key steps. Following  other SARAH-based algorithms, we adopt a similar gradient approximation design with nested loops, specifically with a checkpoint gradient estimator $\bdnu^{(j)}_0$ using a large mini-batch size $B_j$ in Line 4 and a recursive gradient estimator $\bdnu_{k}^{(j)}$ updated in Line 9. When the mini-batch size $B_j$ is large, we can regard the checkpoint gradient estimator $\bdnu^{(j)}_0$ as a solid approximation to the true gradient at $ \tilde{\bdx}_{j-1}$. With this checkpoint, we can update the gradient estimator $\bdnu_{k}^{(j)}$ with a small mini-batch size $b_j$ while maintaining a desirable estimation accuracy.

To emphasize, our stopping rules in Line 11 of Algorithm~\ref{algo1} is newly proposed, which ensures a critical enhancement of the performance compared to previous literature. In particular, with this new design, we can control the gradient norm of the output with high probability. For a more intuitive understanding of these stopping rules, we will see in our proof sketch section that the gradient norm of iterates in the $j$-th outer iteration, $\|\nabla f\|$, can be bounded by a linear combination of $\big\{\bdnu_k\upj\big\}_{k=1}^{K_j}$ with a small remainder. The first stopping rule, therefore, strives to control the magnitude of the linear combination of $\big\{\bdnu_k\upj\big\}_{k=1}^{K_j}$, while the second stopping rule is specifically designed to control the size of remainder terms. For this purpose, $\varepsilon_j$ should be set as a credible controller of the remainder term, with an example given in Theorems \ref{thm:main1}. In this way, with small preset constants $\tilde{\varepsilon}$ and $\varepsilon$, we guarantee that the output has a desirably small gradient norm, dependent on $\tilde{\varepsilon}$ and $\varepsilon$, when the designed stopping rules are activated. Indeed, Proposition \ref{prop:stopgua} in Appendix \ref{append:stopgua} offers a guarantee that the stopping rule will be definitively satisfied at some point. More refined quantitative results regarding the number of steps required for stopping will follow in Theorems \ref{thm:main1} and Appendix \ref{append:sketch_stopping}.


\begin{algorithm}[t!]
   \caption{Probabilistic Stochastic Recursive Gradient (Prob-SARAH)}
\begin{algorithmic}[1]
   \STATE {\bfseries Input:} sample size $n$, constraint area $\D$, initial point $\tilde{\bdx}_0 \in \D$, large batch size $\{B_j\}_{j\ge 1}$, mini batch size  $\{b_j\}_{j\ge 1}$, inner loop length $\{K_j\}_{j\ge 1}$, auxiliary error estimator $\{\varepsilon_j\}_{j\ge 1}$, errors $\tilde{\varepsilon}^2,\varepsilon^2$
   \FOR{$j=1,2,\ldots$}
   \STATE Uniformly sample a batch $\III_j\subseteq\{1,\ldots,n\}$ without replacement, $|\III_j|=B_j$;
   
   \STATE $\bdnu^{(j)}_0 \, \leftarrow \, \frac{1}{B_j}\sum_{i\in \III_j}\nabla f_{i}( \tilde{\bdx}_{j-1})$;
   
   \STATE  $ \bdx_0^{(j)}\,\leftarrow\, \tilde{\bdx}_{j-1}$;
   
   \FOR{$k=1,2,\ldots,K_j$}
    
   \STATE $\bdx_{k}^{(j)}\,\leftarrow\,\mathrm{Proj}\big(\bdx_{k-1}^{(j)}-\eta_j\bdnu_{k-1}^{(j)},\D\big)$, project the update back to $\D$;
   
   \STATE Uniformly sample a mini-batch $\III_{k}^{(j)}\subseteq\{1,\ldots,n\}$ with replacement and $|\III_{k}^{(j)}|=b_j$;
   
   \STATE $\bdnu_{k}^{(j)}\,\leftarrow\,\frac{1}{b_j}\sum_{i\in \III_{k}^{(j)}}\nabla f_i(\bdx_{k}^{(j)})-\frac{1}{b_j}\sum_{i\in \III_{k}^{(j)}}\nabla f_i(\bdx_{k-1}^{(j)})+\bdnu_{k-1}^{(j)}$;
   \ENDFOR
   \IF{$\frac{1}{K_j}\sum_{k=0}^{K_j-1}\big\| \bdnu_k\upj\big\|^2 \leq \tilde{\varepsilon}^2$ and $\varepsilon_j \leq \frac{1}{2}\varepsilon^2$}
   
   \STATE $\hat{k}\,\leftarrow\,\mathop{\arg\min}_{0\leq k \leq K_j-1}\big\| \bdnu_k\upj\big\|^2$;
   
   \STATE \textbf{Return} $\hat{\bdx}\,\leftarrow\,\bdx_{\hat{k}}\upj$;
   \ENDIF
   \STATE $\tilde{\bdx}_{j}\,\leftarrow\,\bdx_{K_j}\upj$;
   \ENDFOR
\end{algorithmic}
\label{algo1}
\end{algorithm}

\section{Theoretical Results}
\label{sec:theory}
This section is devoted to the main theoretical result of our proposed algorithm Prob-SARAH. We provide the stop guarantee of the algorithm along with the upper bound of the steps. The high-probability error bound of the estimated gradient is also established. The discussion of the dependence of our algorithm on the parameters is available after we introduce our main theorems.

\subsection{Technical Assumptions} \label{subsec:assumption}
We shall introduce some necessary regularized assumptions. Most assumptions are commonly used in the optimization literature. We have further clarifications in Appendix \ref{append:assumption}.

\begin{assumption}[Existence of achievable minimum]
\label{assump:minimumavailable}
Assume that for each $i=1,2,\ldots,n$, $f_i$ has continuous gradient on $\D$ and $\D$ is a compact subset of $\RR^d$. Then, there exists a constant $\alpha_M <\infty$ such that
\begin{equation}
    \label{eq_bound_alphaM}
    \max\limits_{1\leq i\leq n}\sup\limits_{\bdx\in \D}\|\nabla f_i(\bdx)\|\leq \alpha_M.
\end{equation}
Also, assume that there exists an interior point $\bdx^*$ of the set $\D$ such that
$$f(\bdx^*) = \inf\limits_{\bdx\in \D} f(\bdx).$$
\end{assumption}


\begin{assumption}[$L$-smoothness]\label{assump:smooth}
For each $i=1,2,\ldots,n$, $f_i:\D\rightarrow \RR$ is $L$-smooth for some constant $L>0$, i.e.,
$$\|\nabla f_i(\bdx) - \nabla f_i(\bdx')\| \leq L\|\bdx -\bdx'\|,\ \forall\ \bdx,\bdx'\in \D.$$
\end{assumption}




\begin{assumption}[$L$-smoothness extension]\label{assump:extent}
There exists a $L$-smooth function $\tilde{f}:\D\rightarrow \RR$ such that
$$\tilde{f}(\bdx) = f(\bdx),\ \forall\ \bdx\in \D, \quad \text{and} \quad \tilde{f}(\mathrm{Proj}(\bdx,\D)) \leq \tilde{f}(\bdx),\ \forall\ \bdx\in \RR^d,$$
where $\mathrm{Proj}(\bdx,\D)$ is the Euclidean projection of $\bdx$ on some compact set $\D$.
\end{assumption}

\begin{assumption}
Assume that the following conditions hold.
\begin{enumerate}
    \item $\varepsilon \leq \frac{1}{e}$ and $\alpha_M^2 \ge \frac{1}{10240}$, where $\epsilon$ is the target error bound in (\ref{eq_probability_bound}) and $\alpha_M$ is defined in (\ref{eq_bound_alphaM}).
    \item The diameter of $\D$ is at least 1, i.e. $d_1 \triangleq \max\{\|\bdx-\bdx'\|:\bdx,\bdx'\in \D\}\ge 1$.
\end{enumerate}
\label{assump:technical}
\end{assumption}

Assumption \ref{assump:minimumavailable} also indicates that there exists a positive number $\Delta_f$ such that
$\sup_{\bdx\in \D} \big[f(\bdx) - f(\bdx^*)\big] \leq \Delta_f.$ Assumptions \ref{assump:minimumavailable}--\ref{assump:extent} are commonly used in the optimization literature, and Assumption \ref{assump:technical} can be easily satisfied in practical use as long as the initial points are not too far from the optimum. See more comments on assumptions in Appendix \ref{append:assumption}.

\subsection{Main Results on Complexity}

According to the definition given in \textcite{lei2020adaptivity}, an algorithm is called $\varepsilon$-independent if it can guarantee convergence at all target accuracies $\varepsilon$ in expectation without explicitly using $\varepsilon$ in the algorithm. This is a very favorable property because it means that we no longer need to set the target error beforehand. Here, we introduce a similar property regarding the dependency on $\varepsilon$.
\begin{definition}[$\varepsilon$-semi-independence]
An algorithm is $\varepsilon$-semi-independent, given $\delta$, if it can guarantee convergence at all target accuracies $\varepsilon$ with probability at least $\delta$ and the knowledge of $\varepsilon$ is only needed in the post-processing. That is, the algorithm can iterate without knowing $\varepsilon$ and we can select an appropriate iterate out afterwards.
\label{def:semi}
\end{definition}

The newly introduced property can be perceived as the probabilistic equivalent of $\varepsilon$-independence. As stated in the succeeding theorem, under the given conditions, Prob-SARAH can achieve $\varepsilon$-semi-independence, given $\delta$.

\begin{theorem}
Suppose that Assumptions \ref{assump:minimumavailable}, \ref{assump:smooth}, \ref{assump:extent} and \ref{assump:technical} are valid. Given a pair of errors $(\varepsilon,\delta)$, in Algorithm \ref{algo1} (Prob-SARAH), set hyperparameters
\begin{equation}
    \label{eq_paraset1_copy}
\scalebox{0.92}{$
    \eta_j = \frac{1}{4L}, \quad K_j = \sqrt{B_j}= \sqrt{j^2\wedge n}, \quad  b_j = l_jK_j, \quad 
\varepsilon_j = 8L^2\tau_j + 2q_j, \quad  \tilde{\varepsilon}^2 = \frac{1}{5}\varepsilon^2,
$}
\end{equation}
for $j\ge1$, where
\begin{align*}
\scalebox{0.9}{$
\tau_j = \frac{1}{j^3}, \delta'_j = \frac{\delta}{4C_ej^4}, \quad
l_j = 18\Big( \log (\frac{2}{\delta'_j})+\log\log (\frac{2d_1}{\tau_j})\Big), \quad
q_j = \frac{128\alpha_M^2}{B_j}\log(\frac{3}{\delta'_j})\mathbf{1}\left\{B_j<n\right\}.
$}
\end{align*}
Then,
$$
Comp(\varepsilon,\delta) = \tilde{\OOO}_{L,\Delta_f,\alpha_M}\Big(\frac{1}{\varepsilon^3}\wedge \frac{\sqrt{n}}{\varepsilon^2}\Big),
$$
where $Comp(\varepsilon,\delta)$ represents the number of computations needed to get an output $\hat{\bdx}$ satisfying $\left\|\nabla f(\hat{\bdx})\right\|^2 \leq \varepsilon^2$ with probability at least $1-\delta$.
\label{thm:main1}
\end{theorem}

More detailed results can be found in Appendix \ref{append:complexity}. In appendix \ref{append:complexity}, we also introduce another hyper-parameter setting that can lead to a complexity with better dependency on $\alpha_M^2$, which could be implicitly affected by the choice of constraint region $\D$.



\subsection{Proof Sketch}
In this part, we explain the idea of the proof of Theorem \ref{thm:main1}. Same proofing strategy can be applied to other hyper-parameter settings. First, we bound the difference between $\bdnu_k\upj$ and $\nabla f\big( \bdx_k\upj\big)$ by a linear combination of $\{\|\bdnu_{m}\upj\|\}_{m=0}^{k-1}$ and small remainders, with which we can have a good control on $\|\nabla f\big( \bdx_k\upj\big)\|$ when the stopping rules are met. Second, we bound the number of steps we need to meet the stopping rules. Combining these 2 key components, we can smoothly get the final conclusions.

Let us firstly introduce a novel Azuma-Hoeffding type inequality, which is key to our analysis.
\begin{theorem}[Martingale Azuma-Hoeffding Inequality with Random Bounds]
Suppose $\bdz_1,\ldots,\bdz_K\in \RR^d$ is a martingale difference sequence adapted to $\FFF_0,\ldots,\FFF_K$. Suppose $\{r_k\}_{k=1}^K$ is a sequence of random variables such that $\|\bdz_k\|\leq r_k$ and $r_k$ is measurable with respect to $\FFF_k$, $k=1,\ldots,K$. Then, for any fixed $\delta>0$, and $B>b>0$, with probability at least $1-\delta$, for $1\leq t\leq K$, either
\begin{equation*}
\scalebox{0.9}{
$
\exists 1\leq t\leq K,\ \sum\limits_{k=1}\limits^{t}r_k^2 \ge B \text{ or }
\Big\| \sum\limits_{k=1}\limits^{t}\bdz_k\Big\|^2 \leq 9\max\Big\{ \sum\limits_{k=1}\limits^{t}r_k^2,b\Big\}\Big(\log (\frac{2}{\delta})+\log \log (\frac{B}{b})\Big).
$}
\end{equation*}
\label{thm:martingaleah}
\end{theorem}
\vspace{-1cm}
\begin{remark}
    It is noteworthy that this probabilistic bound on large-deviation is dimension-free, which is a nontrivial extension of Theorem 3.5 in \textcite{pinelis1994optimum}. If $r_1,r_2,\ldots,r_K$ are not random, we can let $B=\sum_{k=1}^{K}r_k^2 + \zeta_1$ and $b=\zeta_2B$ with $\zeta_1>0$, $0<\zeta_2<1$. Since $\zeta_1$ can be arbitrarily close to 0 and  $\zeta_2$ can be arbitrarily close to 1, we can recover Theorem 3.5 in \textcite{pinelis1994optimum}. Compared with Corollary 8 in \textcite{jin2019short}, which can be viewed as a sub-Gaussian counterpart of our result, a key feature of our Theorem \ref{thm:martingaleah} is its dimension-independence. We are also working towards improving the bound in Corollary 8 from \textcite{jin2019short} to a dimension-free one.
\end{remark}

The success of Algorithm \ref{algo1} is largely because $\nabla f(\bdx_k\upj)$ is well-approximated by $\bdnu_k\upj$, and meanwhile $\bdnu_k\upj$ can be easily updated. We can observe that $\bdnu_k\upj-\nabla f(\bdx_k\upj)$ is actually sum of a sequence of martingale difference as
{\footnotesize
\begin{align}
    & \bdnu_k\upj-\nabla f(\bdx_k\upj) 
    = \Big[\frac{1}{b_j}\sum_{i\in \III_k\upj}\nabla f_i(\bdx_k\upj)   - \frac{1}{b_j}\sum_{i\in \III_k\upj}\nabla f_i(\bdx_{k-1}\upj) \nonumber  +\nabla f(\bdx_{k-1}\upj)- \nabla f(\bdx_k\upj) \Big] \nonumber \\
    & \quad+ \left[\bdnu_{k-1}\upj - \nabla f(\bdx_{k-1}\upj)\right] = \sum_{m=1}^k \Big[\frac{1}{b_j}\sum_{i\in \III_m\upj}\nabla f_i(\bdx_m\upj)  - \frac{1}{b_j}\sum_{i\in \III_m\upj}\nabla f_i(\bdx_{m-1}\upj) \nonumber \\
    & \quad +\nabla f(\bdx_{m-1}\upj) - \nabla f(\bdx_m\upj) \Big]+ \left[\bdnu_{0}\upj - \nabla f(\bdx_{0}\upj)\right].
\end{align}
}
To be more specific, let $\FFF_0 = \{\emptyset, \Omega\}$, and iteratively define $\FFF_{j,-1} = \FFF_{j-1}$, $\FFF_{j,0} = \sigma\big(\FFF_{j-1}\cup \sigma(\III_j)\big)$, $\FFF_{j,k}=\sigma\big(\FFF_{j,0}\cup \sigma( \III_{k}\upj)\big)$, $\FFF_j=\sigma\big(\bigcup\limits_{k=1}\limits^{\infty}\FFF_{j,k}\big)$, $j\ge 1,k\ge 1$. We also denote $\bdepsilon_0\upj\triangleq \bdnu_0\upj-\nabla f\big(\bdx_0\upj\big)$, $\bdepsilon_m\upj \triangleq \frac{1}{b_j}\sum\limits_{i\in \III_m\upj}\nabla f_i(\bdx_m\upj) - \nabla f(\bdx_m\upj) +\nabla f(\bdx_{m-1}\upj) - \frac{1}{b_j}\sum\limits_{i\in \III_m\upj}\nabla f_i(\bdx_{m-1}\upj)$, $m\ge 1$. Then, we can see that $\{\bdepsilon_m\upj\}_{m=0}^{k}$ is a martingale difference sequence adapted to $\{ \FFF_{j,m}\}_{m=-1}^{k}$. With the help of our new Martingale Azuma-Hoeffding inequality, we can control the difference between $\bdnu_k\upj$ and $\nabla f\big( \bdx_k\upj\big)$ by a linear combination of $\{\|\bdnu_{m}\upj\|\}_{m=0}^{k-1}$ and small remainders, with details given in Appendix \ref{append:sketch_diff}. Then, given the stopping rules in line 11 and selection method specified in line 12 of Algorithm \ref{algo1}, it would be not hard for us to obtain $\left\|\nabla f(\hat{\bdx})\right\|^2 \leq \varepsilon^2$ with a high probability. More details can be found in Appendix \ref{append:sketch_output}.


Another key question needed to be resolved is, when the algorithm can stop? The following analysis can build some intuitions for us. Given a $T\in \ZZ_+$, with the bound given in Proposition \ref{prop:innerloop} in Appendix \ref{append:proof}, with a high probability,
\begin{align}
-\Delta_f &\leq f\left( \tilde{\bdx}_{{2T}}\right) - f\left( \tilde{\bdx}_{{T}}\right) \leq A_T - \frac{1}{16L}\sum\limits_{j=T+1}\limits^{2T}\sum\limits_{k=0}\limits^{K_j-1} \left\| \bdnu_k\upj\right\|^2,
\label{eq:fvaluebound2}
\end{align}
where $A_T$ is upper bounded by a value polylogorithmic in $T$. As for the second summation, if $\varepsilon_j \leq \frac{1}{2}\varepsilon^2$ for $j=T,T+1,\ldots,2T$ (which is obviously true when $T$ is moderately large) and our algorithm doesn't stop in $2T$ outer iterations,
{\small
\begin{align*}
&\quad \frac{1}{16L}\sum\limits_{j=T+1}\limits^{2T} \sum\limits_{k=0}\limits^{K_j-1} \left\| \bdnu_k\upj\right\|^2 \ge \frac{\tilde{\varepsilon}^2}{16L} \sum\limits_{j=T+1}\limits^{2T}K_j\nonumber  \ge \frac{\tilde{\varepsilon}^2}{16L} \sum\limits_{j=T+1}\limits^{2T} \left( T\wedge \sqrt{n} \right) = \frac{\tilde{\varepsilon}^2}{16L}T^2 \wedge (\sqrt{n}T),
\end{align*}
}
which grows at least linear in $T$. Consequently, when $T$ is sufficiently large, the RHS of (\ref{eq:fvaluebound2}) can be smaller than $-\Delta_f$, which leads to a contradiction. Roughly, we can see that the stopping time $T$ cannot exceed the order of $\tilde{\OOO} \big( \frac{1}{\varepsilon} \vee \frac{1}{\sqrt{n}\varepsilon^2}\big)$. More details can be found in Appendix \ref{append:sketch_stopping}.
%


\section{Numerical Experiments}
\label{sec:numerical}

In order to validate our theoretical results and show good probabilistic property for the newly-introduced Prob-SARAH, we conduct some numerical experiments where the objectives are possibly non-convex.

\subsection{Logistic Regression with Non-Convex Regularization}

In this part, we consider to add a non-convex regularization term to the commonly-used logistic regression. Specifically, given a sequence of observations $(\bdw_i,y_i)\in \RR^d \times \{-1,1\}$, $i=1,2,\ldots,n$ and a regularized parameter $\lambda>0$, the objective is
$$
f(\bdx) = \frac{1}{n}\sum\limits_{i=1}\limits^{n} \log \left(1+e^{-y_i\bdw_i^T\bdx}\right) + \frac{\lambda}{2}\sum\limits_{j=1}\limits^{d}\frac{x_j^2}{1+x_j^2}.
$$
\begin{figure*}[t]
     \centering
     \begin{subfigure}[b]{0.24\textwidth}
         \centering
         \includegraphics[width=\textwidth]{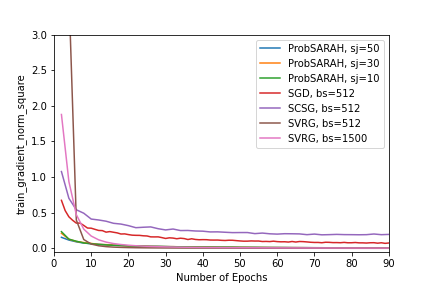}
         \label{fig:g1}
     \end{subfigure}
     \hfill
     \begin{subfigure}[b]{0.24\textwidth}
         \centering
         \includegraphics[width=\textwidth]{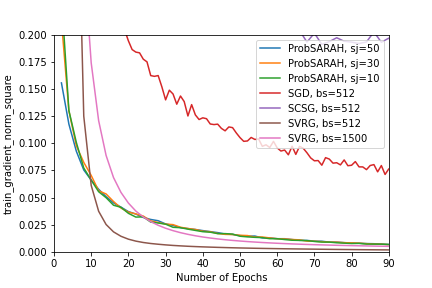}
         \label{fig:g2}
     \end{subfigure}
     \hfill
     \begin{subfigure}[b]{0.24\textwidth}
         \centering
         \includegraphics[width=\textwidth]{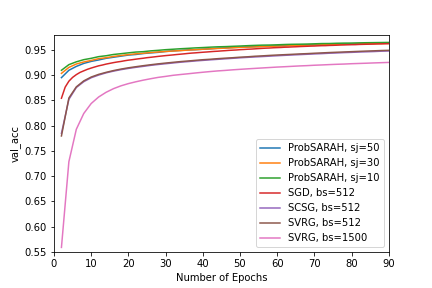}
         \label{fig:g3}
     \end{subfigure}
     \hfill
     \begin{subfigure}[b]{0.24\textwidth}
         \centering
         \includegraphics[width=\textwidth]{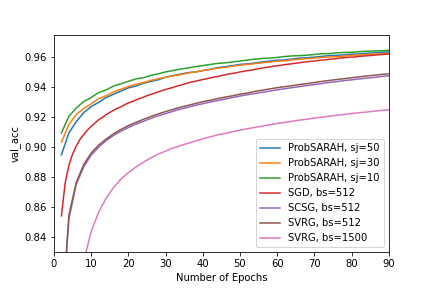}
         \label{fig:g4}
     \end{subfigure}
     
     \begin{subfigure}[b]{0.24\textwidth}
         \centering
         \includegraphics[width=\textwidth]{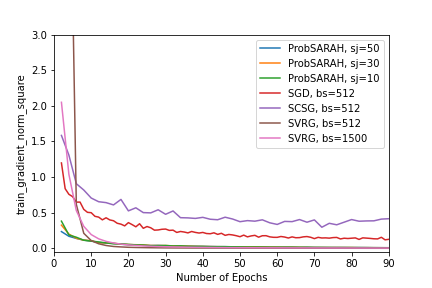}
         \label{fig:g5}
     \end{subfigure}
     \hfill
     \begin{subfigure}[b]{0.24\textwidth}
         \centering
         \includegraphics[width=\textwidth]{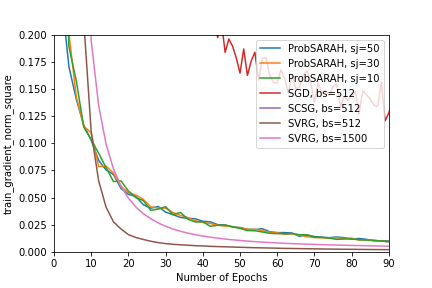}
         \label{fig:g6}
     \end{subfigure}
     \hfill
     \begin{subfigure}[b]{0.24\textwidth}
         \centering
         \includegraphics[width=\textwidth]{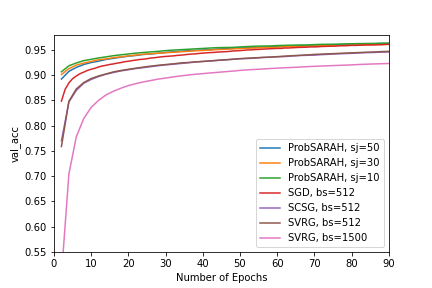}
         \label{fig:g7}
     \end{subfigure}
     \hfill
     \begin{subfigure}[b]{0.24\textwidth}
         \centering
         \includegraphics[width=\textwidth]{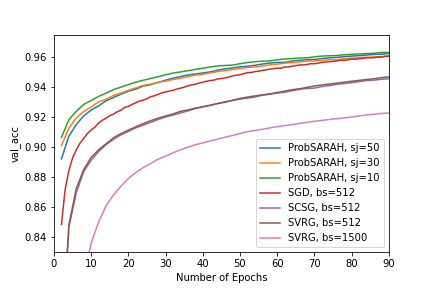}
         \label{fig:g8}
     \end{subfigure}
     \caption{Comparison of convergence with respect to $(1-\delta)$-quantile of square of gradient norm $\left( \|\nabla f\|^2\right)$ and $\delta$-quantile of validation accuracy on the \textbf{MNIST} dataset for $\delta=0.1$ and $\delta=0.01$. The second (fourth) column presents zoom-in figures of those in the first (third) column. Top: $\delta=0.1$. Bottom: $\delta=0.01$. 'bs' stands for batch size. 'sj=x' means that the smallest batch size $\approx x\log x$.}
     \label{fig:mnist1}
\end{figure*}

Such an objective has also been considered in other works like \textcite{horvath2020adaptivity} and \textcite{ji2020history}. Same as other works, we set the regularized parameter $\lambda=0.1$ across all experiments. We compare the newly-introduced Prob-SARAH against three popular methods including SGD (\cite{ghadimi2013stochastic}), SVRG (\cite{reddi2016stochastic}) and SCSG (\cite{lei2017non}). Based on results given in Theorem \ref{thm:main1}, we let the length of the inner loop $K_j \sim j\wedge \sqrt{n}$, the inner loop batch size $b_j \sim \log j \left( j \wedge \sqrt{n} \right)$, the outer loop batch size $B_j \sim j^2 \wedge n$. For fair comparison, we determine the batch size (inner loop batch size) for SGD (SCSG and SVRG) based on the sample size $n$ and the number of epochs needed to have sufficient decrease in gradient norm. For example, for the w7a dataset, the sample size is 24692 and we run 60 epochs in total. In the 20th epoch, the inner loop batch size of Prob-SARAH is approximately $67\log 67 \approx 281$. Thus, we set batch size 256 for SGD, SCSG and SVRG so that they can be roughly matched. In addition, based on the theoretical results from \textcite{reddi2016stochastic}, we also consider a large inner loop batch size comparable to $n^{2/3}$ for SVRG. In addition, we set step size $\eta = 0.01$ for all algorithms across all experiments for simplicity.

Results are displayed in Figure \ref{fig:logistic_gradient}, from which we can see that Prob-SARAH has superior probabilistic guarantee in controlling the gradient norm in all experiments. It is significantly better than SCSG and SVRG under our current setting. Prob-SARAH can achieve a lower gradient norm than SGD at the early stage while SGD has a slight advantage when the number of epochs is large.

\begin{figure*}[htbp!]
     \centering
     \begin{subfigure}[b]{0.32\textwidth}
         \centering
         \includegraphics[width=\textwidth, height = 0.6\textwidth]{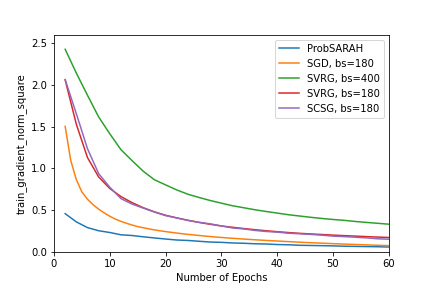}
         \label{fig:m01}
     \end{subfigure}
     \hfill
     \begin{subfigure}[b]{0.32\textwidth}
         \centering
         \includegraphics[width=\textwidth, height = 0.6\textwidth]{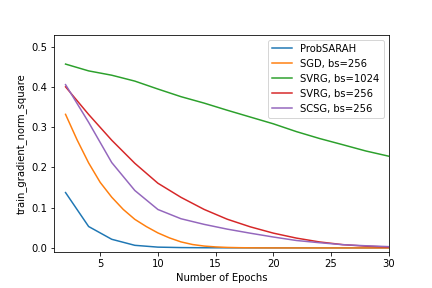}
         \label{fig:i01}
     \end{subfigure}
     \hfill
     \begin{subfigure}[b]{0.32\textwidth}
         \centering
         \includegraphics[width=\textwidth, height = 0.6\textwidth]{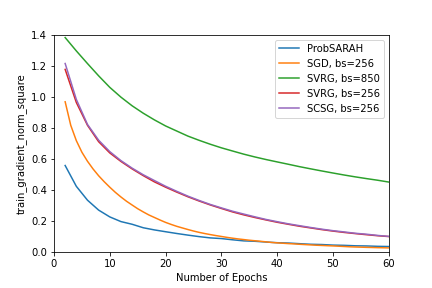}
         \label{fig:w01}
     \end{subfigure}
    \par \vspace{-1.5em}
    \begin{subfigure}[b]{0.32\textwidth}
         \centering
         \includegraphics[width=\textwidth, height = 0.6\textwidth]{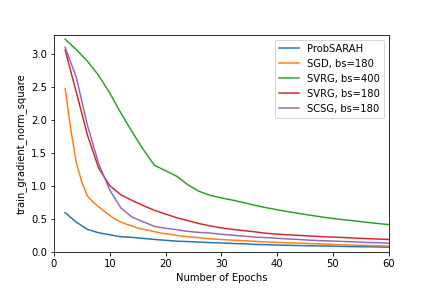}
         \label{fig:m02}
     \end{subfigure}
     \hfill
     \begin{subfigure}[b]{0.32\textwidth}
         \centering
         \includegraphics[width=\textwidth, height = 0.6\textwidth]{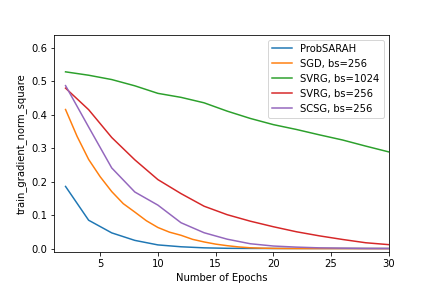}
         \label{fig:i02}
     \end{subfigure}
     \hfill
     \begin{subfigure}[b]{0.32\textwidth}
         \centering
         \includegraphics[width=\textwidth, height = 0.6\textwidth]{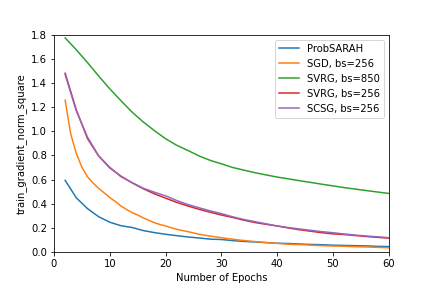}
         \label{fig:w02}
     \end{subfigure}
     \vspace{-1.5em}
     \caption{Comparison of convergence with respect to $(1-\delta)$-quantile of square of gradient norm $\left(\|\nabla f\|^2\right)$ over 3 datasets for $\delta=0.1$ and $\delta=0.01$. Top: $\delta=0.1$. Bottom: $\delta=0.01$. Datasets: \textbf{mushrooms}, \textbf{ijcnn1}, \textbf{w7a} (from left to right). 'bs' stands for batch size.}
        \label{fig:logistic_gradient}
\end{figure*}


\subsection{Two-Layer Neural Network}

We also evaluate the performance of Prob-SARAH, SGD, SVRG and SCSG on the MNIST dataset with a simple 2-layer neural network. The two hidden layers respectively have 128 and 64 neurons. We include a GELU activation layer following each hidden layer. We use the negative log likelihood as our loss function. Under this setting, the objective is possibly non-convex and smooth on any given compact set. The step size is fixed to be 0.01 for all algorithms. For Prob-SARAH, we still have the length of the inner loop $K_j \sim j\wedge \sqrt{n}$, the inner loop batch size $b_j \sim \log j \left( j \wedge \sqrt{n} \right)$, the outer loop batch size $B_j \sim j^2 \wedge n$. But to reduce computational time, we let $j$ start from 10, 30 and 50 respectively. Based on the  same rule described in the previous subsection, we let the batch size (or inner loop batch size) for SGD, SVRG and SCSG be 512.

Results are given in Figure \ref{fig:mnist1}. In terms of gradient norm, Prob-SARAH has the best performance among algorithms considered here when the number of epochs is relatively small. With increasing number of epochs, SVRG tends to be better in finding first-order stationary points. However, based on the 3rd and 4th columns in Figure \ref{fig:mnist1}, SVRG apparently has an inferior performance on the validation set, which indicates that it could be trapped at local minima. In brief, Prob-SARAH achieves the best tradeoff between finding a first-order stationary point and generalization.

We also consider another set of experiments by replacing the GELU activation function with ReLU, resulting in a non-smooth objective. The results are shown in Appendix \ref{append:fig}, which resemble those in Figure \ref{fig:mnist1} and the similar conclusions can be drawn.

\section{Conclusion}
In this paper, we propose a SARAH-based variance reduction algorithm called Prob-SARAH and provide high-probability bounds on gradient norm for estimator resulted from Prob-SARAH. Under appropriate assumptions, the high-probability first order complexity nearly match the one in the in-expectation sense. The main tool used in the theoretical analysis is a novel Azuma-Hoeffding type inequality. We believe that similar probabilistic analysis can be applied to SARAH-based algorithms in other settings.

\begin{appendix}
\section{Remarks and Examples for Assumptions}
\label{append:assumption}

\subsection{More comments on Assumptions \ref{assump:minimumavailable}--\ref{assump:technical}}
\begin{remark}[Convexity and smoothness]
It is worth noticing that Assumption \ref{assump:minimumavailable} is widely used in many non-convex optimization works and can be met for most applications in practice. Assumption \ref{assump:smooth} is also needed in deriving in-expectation bound for many non-convex variance-reduced methods, including state-of-art ones like SPIDER and SpiderBoost. As for Assumption \ref{assump:extent}, it is a byproduct of the compact constraint and can be satisfied with some commonly-seen $f$ and usual choices of $\D$. For more discussions on Assumption \ref{assump:extent}, please see Appendix \ref{sec:assextent}.
\end{remark}

\begin{remark}[Compact set $\D$]
    Compared with other works in the literature of non-convex optimization, the compact constraint region $\D\in \RR^d$ imposed in the finite sum problem (\ref{mainproblem}) may seem somewhat restrictive. In fact, such constraint is largely due to technical convenience and it can be removed with additional condition on gradients. We will elaborate on this point in subsection \ref{subsec:param}. Besides, in many practical applications, it is reasonable to restrict estimators to a compact set when certain prior knowledge is available.
\end{remark}

\subsection{An Example of Assumption \ref{assump:extent}}
\label{sec:assextent}
Let us consider the logistic regression with non-convex regularization where the object function can be characterized as
$$
f(\bdx) = \frac{1}{n}\sum\limits_{i=1}\limits^n \log \left( 1+\text{exp}\left( -y_i\langle \bdw_i,\bdx\rangle \right) \right) + \frac{\lambda}{2}\Phi(\bdx),
$$
where $\Phi(\bdx)=\sum\limits_{j=1}\limits^d (x_j^2)^{\frac{1}{4}}$, $x_j$ is the $j$th element of $\bdx$, $\lambda>0$ is the regularization parameter, $\{y_i\}_{i=1}^n$ are labels and $\{\bdw_i\}_{i=1}^n$ are normalized covariates with norm 1. In fact, for any fixed $\lambda>0$, Assumption \ref{assump:extent} holds with $\tilde{f}=f$ and $\D=\{\bdx:\|\bdx\|\leq R\}$ when $R$ is sufficiently large. Since smoothness is easy to show, we focus on the second part of Assumption \ref{assump:extent}. To show that
$$
f\left( \mathrm{Proj}(\bdx,\D)\right) \leq f(\bdx)
$$
holds for any $\bdx \in \RR^d$, since the projection direction is pointed towards the origin, it suffices to show that for any $\bdnu \in \RR^d$ with $\|\bdnu \|=1$,
$$
\frac{d}{dt}f_i(t\bdnu) = \frac{d}{dt}\Big( \log \big( 1+\text{exp}( -ty_i\langle \bdw_i,\bdnu\rangle ) \big) + \frac{\lambda}{2}\sum\limits_{j=1}\limits^d \sqrt{t} (\nu_j^2)^{\frac{1}{4}} \Big) \ge 0,
$$
when $t\ge R$ for $i=1,2,\ldots,n$, where $\nu_j$ is the $j$th element of $\bdnu$. To see this,
\begin{align*}
&\quad \frac{d}{dt}f_i(t\bdnu)\\
&= \frac{-y_i\langle \bdw_i,\bdnu\rangle \text{exp}\left( -ty_i\langle \bdw_i,\bdnu\rangle \right)}{1+\text{exp}\left( -ty_i\langle \bdw_i,\bdnu\rangle \right)} + \frac{\lambda}{2}\sum\limits_{j=1}\limits^d \frac{(\nu_j^2)^{\frac{1}{4}}}{2\sqrt{t}}\\
&= \frac{-y_i\langle \bdw_i,\bdnu\rangle }{1+\text{exp}\left( ty_i\langle \bdw_i,\bdnu\rangle \right)} + \frac{\lambda}{2}\sum\limits_{j=1}\limits^d \frac{(\nu_j^2)^{\frac{1}{4}}}{2\sqrt{t}}\\
&\ge \frac{-y_i\langle \bdw_i,\bdnu\rangle }{1+\text{exp}\left( ty_i\langle \bdw_i,\bdnu\rangle \right)} + \frac{\lambda}{2} \sum\limits_{j=1}\limits^d \frac{\nu_j^2}{2\sqrt{t}}\\
&= \frac{-y_i\langle \bdw_i,\bdnu\rangle }{1+\text{exp}\left( ty_i\langle \bdw_i,\bdnu\rangle \right)} + \frac{\lambda}{4\sqrt{t}}.
\end{align*}

If $y_i\langle \bdw_i,\bdnu\rangle \leq 0$, we can immediately know that $\frac{d}{dt}f_i(t\bdnu) \ge 0$ for any $t > 0$.

If $y_i\langle \bdw_i,\bdnu\rangle > 0$, let us consider an auxiliary function
$$
g(b) = \frac{-b}{1+e^{tb}}.
$$
Then,
$$
g'(b) \propto -\left( 1+e^{tb}\right) +bt e^{bt},
$$
from where we can know the minimum of $g(b)$ is achieved for some $b^*\in [\frac{1}{t},\frac{2}{t}]$. Thus, 
$$
g(b)\ge g(b^*) \ge \frac{-2}{(1+e^{tb})t} \ge \frac{-2}{(1+e)t}.
$$
Therefore,
$$
\frac{d}{dt}f_i(t\bdnu) \ge \frac{-2}{(1+e)t} + \frac{\lambda}{4\sqrt{t}},
$$
which is positive when $t\ge\left( \frac{8}{(1+e)\lambda}\right)^2$.

If we consider other non-convex regularization terms in logistic regression, such as $\Phi(\bdx)=\sum\limits_{j=1}\limits^d \frac{x_j^2}{1+x_j^2}$, we may no longer enjoy Assumption \ref{assump:extent} because monotony may not hold for a few projection directions even when the constraint region is large. Nevertheless, such theoretical flaw can be easily remedied by adding an extra regularization term like $\frac{\lambda_e}{2}\|\bdx\|^2$ with appropriate $\lambda_e>0$.

\section{Stop Guarantee}
\label{append:stopgua}
We would like to point out that, under appropriate parameter setting, Prob-SARAH is guaranteed to stop. Actually, we can have the stopping guarantee under more general conditions than those stated in the following proposition. But for simplicity, we only present conditions naturally matched parameter settings given in the next two subsections.

\begin{proposition}[Stop guarantee of Prob-SARAH]\label{prop:stopgua}
Suppose that Assumptions \ref{assump:minimumavailable}, \ref{assump:smooth}, \ref{assump:extent} and \ref{assump:technical} are satisfied. Let step size $\eta_j\equiv 1/(4L)$ and suppose that $b_j \ge K_j$, $j\ge1$. The large batch size $\{B_j\}_{j\ge1}$ is set appropriately such that $B_j=n$ when $j$ is sufficiently large. If the limit of $\{\varepsilon_j\}_{j\ge1}$ is 0, then, for any fixed $\tilde{\varepsilon}$ and $\varepsilon$, with probability 1, Prob-SARAH (Algorithm \ref{algo1}) stops. In settings where we always have $\varepsilon_j \leq \frac{1}{2}\varepsilon^2$, we also have the result that Prob-SARAH (Algorithm \ref{algo1}) stops with probability 1.
\end{proposition}


\section{Detailed Results on Complexity}
\label{append:complexity}

\begin{theorem}
Suppose that Assumptions \ref{assump:minimumavailable}, \ref{assump:smooth}, \ref{assump:extent} and \ref{assump:technical} are valid. Given a pair of errors $(\varepsilon,\delta)$, in Algorithm \ref{algo1} (Prob-SARAH), set hyperparameters
\begin{equation}
    \label{eq_paraset1}
\scalebox{0.92}{$
    \eta_j = \frac{1}{4L}, \quad K_j = \sqrt{B_j}= \sqrt{j^2\wedge n}, \quad  b_j = l_jK_j, \quad 
\varepsilon_j = 8L^2\tau_j + 2q_j, \quad  \tilde{\varepsilon}^2 = \frac{1}{5}\varepsilon^2,
$}
\end{equation}
for $j\ge1$, where
\begin{align*}
\scalebox{0.9}{$
\tau_j = \frac{1}{j^3}, \delta'_j = \frac{\delta}{4C_ej^4}, \quad
l_j = 18\Big( \log (\frac{2}{\delta'_j})+\log\log (\frac{2d_1}{\tau_j})\Big), \quad
q_j = \frac{128\alpha_M^2}{B_j}\log(\frac{3}{\delta'_j})\mathbf{1}\left\{B_j<n\right\}.
$}
\end{align*}
Then, with probability at least $1-\delta$, Prob-SARAH stops in at most
$$
2(T_1 \vee T_2 \vee T_3 \vee T_4) = \tilde{\OOO}_{L,\Delta_f,\alpha_M}\left( \frac{1}{\varepsilon} + \frac{1}{\sqrt{n}\varepsilon^2}\right)
$$
outer iterations and the output satisfies $\left\| \nabla f(\hat{\bdx})\right\|^2 \leq \varepsilon^2$. Detailed definitions of $T_1,T_2,T_3$ and $T_4$ can be found in Propositions \ref{prop:stopyespart1} and \ref{prop:stopyespart2}.
\label{thm:main1_append}
\end{theorem}

\begin{corollary}\label{cor:setting1_append}
Under parameter settings in Theorem \ref{thm:main1_append},
$$
Comp(\varepsilon,\delta) = \tilde{\OOO}_{L,\Delta_f,\alpha_M}\left(\frac{1}{\varepsilon^3}\wedge \frac{\sqrt{n}}{\varepsilon^2}\right).
$$
\end{corollary}

We introduce another setting that can help to reduce the dependence on $\alpha_M^2$, which could be implicitly affected by the choice of constraint region $\D$. We should also notice that, under such setting, the algorithm is no longer $\varepsilon$-semi-independent.

\begin{theorem}
Suppose that Assumptions \ref{assump:minimumavailable}, \ref{assump:smooth}, \ref{assump:extent} and \ref{assump:technical} are valid. We denote $\Delta_f^0 \triangleq f\left( \tilde{\bdx}_0\right) - f\left( \bdx^*\right)$. Given a pair of errors $(\varepsilon,\delta)$, in Algorithm \ref{algo1} (Prob-SARAH), set parameters
\begin{equation}
    \label{eq_paraset2}
    \eta_j = \frac{1}{4L}, \quad K_j = \sqrt{B_j}= \sqrt{n}, \quad  b_j = l_jK_j, \quad
\varepsilon_j = \frac{1}{2}\tilde{\varepsilon}^2= \frac{1}{10}\varepsilon^2,
\end{equation}
for $j\ge1$, where
\begin{align*}
\tau_j = \frac{1}{40L^2}\varepsilon^2, \delta'_j = \frac{\delta}{4C_ej^4}, \quad 
l_j = 18\Big( \log (\frac{2}{\delta'_j})+\log\log (\frac{2d_1}{\tau_j})\Big).
\end{align*}
Then, with probability at least $1-\delta$, Prob-SARAH stops in at most
$$
T_5 = \frac{160L(\Delta_f^0+1)}{\sqrt{n}\varepsilon^2} = \OOO_{L,\Delta_f^0}\left( \frac{1}{\sqrt{n}\varepsilon^2} \right)
$$
outer iterations and the output satisfies $\left\| \nabla f(\hat{\bdx})\right\|^2 \leq \varepsilon^2$.

\label{thm:main2_append}
\end{theorem}

\begin{corollary}
Under parameter settings in Theorem \ref{thm:main2_append},
$$
Comp(\varepsilon,\delta) = \tilde{\OOO}_{L,\Delta_f^0}\left(\frac{\sqrt{n}}{\varepsilon^2}\right).
$$
\label{cor:setting2_append}
\end{corollary}
\vspace{-1.5em}
Comparing the complexities in Corollary \ref{cor:setting1_append} and Corollary \ref{cor:setting2_append}, we can notice that under the second setting, we can get rid of the dependence on $\alpha_M$ at the expense of losing some adativity to $\varepsilon$. We would also like to point out that, in the low precision region, i.e. when $\frac{1}{\varepsilon} = o\left(\sqrt{n}\right)$, the second complexity is inferior.

\subsection{Dependency on Parameters}
\label{subsec:param}

To apply our newly-introduce Azuma-Hoeffding type inequality (see Theorem \ref{thm:martingaleah}), it is necessary to impose a compact constraint region $\D$. Therefore, let us provide a delicate analysis on how $\D$ can affect the convergence guarantee. 

\noindent\textbf{Dependency on $d_1$:} $d_1$, the diameter of $\D$, is a parameter directly related to the choice of $\D$. Shown in theoretical results presented above, the in-probability first-order complexities always have a polylogorithmic dependency on $d_1$, which implies that as long as $d_1$ is polynomial in $n$ or $\frac{1}{\varepsilon}$, it should only have a minor effect on the complexity. With certain prior knowledge, we should be able to control $d_1$ at a reasonable scale.

\noindent\textbf{Dependency on $\Delta_f$ and $\Delta_f^0$:} Under the setting given in Theorem \ref{thm:main1}, the first-order complexity is polynomial in $\Delta_f$. Such dependency implicates that the complexity would not deteriorate much if $\Delta_f$ is of a small order, which is definitely true when the loss function is bounded. As for the setting given in Theorem \ref{thm:main2_append}, the first-order complexity is polynomial in $\Delta_f^0$, which is conventionally assumed to be $\OOO(1)$ and will not be affected by $\D$.

\section{Postponed Proofs for the Results in Section \ref{sec:theory}}
\label{append:sketch}

\subsection{Bounding the Difference between $\bdnu_k\upj$ and $\nabla f\big( \bdx_k\upj\big)$}
\label{append:sketch_diff}

\begin{proposition}
For $k\ge 0$, $j\ge 1$, denote
$$
\big(\tilde{\sigma}_k\upj\big)^2 \triangleq \frac{4L^2\eta_j^2}{b_j}\sum\limits_{m=1}\limits^k \big\| \bdnu_{m-1}\upj\big\|^2.
$$
Under Assumptions \ref{assump:smooth} and \ref{assump:technical}, for any prescribed constant $\delta'\in (0,1)$, $\tau\in(0,1)$, $k\ge 0$, $j\ge 1$,
\begin{align}
&\quad \big\| \bdnu_k\upj - \nabla f\big(\bdx_k\upj\big)\big\|^2 \nonumber\\
&\leq 18\Big(  \big(\tilde{\sigma}_k\upj\big)^2 + \frac{4L^2\tau k}{b_j} \Big)\Big( \log \frac{2}{\delta'} + \log \log \frac{2d_1^2}{\tau}\Big)\label{eq:importantbound1}\\
&\quad+ \frac{128\alpha_M^2}{B_j}\log\frac{3}{\delta'}\mathbf{1}\left\{B_j<n\right\}\nonumber
\end{align}
with probability at least $1-2\delta'$.
\label{prop:nudiff}
\end{proposition}

\begin{remark}
    Let us briefly explain this high-probability bound on $\| \bdnu_k\upj - \nabla f(\bdx_k\upj)\|^2$. When $k=o(b_j)$ and $L=\OOO({1})$, by letting $\tau$ be of appropriate $n^{-1}$-polynomial order, $4L^2\tau k/b_j$ will be roughly $o(1)$. If further we have $d_1$ be of $n$-polynomial order and let $\delta'$ be of $n^{-1}$-polynomial order, $\log (2/\delta') +\log \log (2d_1^2/\tau)$ will be $\tilde{\OOO}(1)$. As a result, the upper bound is roughly $(\tilde{\sigma}_k\upj)^2=(4L^2\eta_j^2/b_j)\sum_{m=1}^{k}\|\bdnu_{m-1}\upj\|^2$ when $B_j$ is sufficiently large so that the last term in the bound (\ref{eq:importantbound1}) is negligible. Bounding $\| \bdnu_k\upj - \nabla f(\bdx_k\upj)\|^2$ by linear combination of $\{\|\bdnu_{m}\upj\|\}_{m=0}^{\infty}$ is the key to our analysis. 
\end{remark}


\subsection{Analysis on the Output $\hat{\bdx}$}
\label{append:sketch_output}

Under parameter setting specified in Theorem \ref{thm:main1}, if we suppose that the algorithm stops at the $j$-th outer iteration, i.e.
\begin{equation}
\frac{1}{K_j}\sum\limits_{k=0}\limits^{K_j-1}\left\|\bdnu_k^{(j)}\right\|^2 \leq \tilde{\varepsilon}^2,\ \varepsilon_j \leq \frac{1}{2}\varepsilon,
\label{eq:stopcond}
\end{equation}
there must exist a $0\leq k'\leq K_j-1$, such that $\left\|\bdnu_{k'}^{(j)}\right\|^2 \leq \tilde{\varepsilon}^2.$

Then, on the event 
\begin{equation}
\Omega_j \triangleq \left\{\omega:\left\| \bdnu_k\upj - \nabla f\left(\bdx_k\upj\right)\right\|^2 \leq l_j\left( \left(\tilde{\sigma}_k\upj\right)^2+\frac{4L^2\tau_jk}{b_j}\right) +q_j,0\leq k\leq K_j\right\},
\label{eq:omegaj}
\end{equation}
where $l_j = 18\left( \log \frac{2}{\delta'_j} + \log \log \frac{2d_1^2}{\tau_j}\right)$ and $q_j=\frac{128\alpha_M^2}{B_j}\log\frac{3}{\delta'_j}\mathbf{1}\left\{B_j<n\right\}$, we can easily derive an upper bound on $\left\| \nabla f(\bdx_{k'}^{(j)})\right\|^2$,
\begin{align*}
    &\quad \left\| \nabla f(\bdx_{k'}^{(j)})\right\|^2\\
    &\leq 2\|\bdnu_{k'}^{(j)}\|^2 + 2\left\| \bdnu_{k'}\upj - \nabla f\left(\bdx_{k'}\upj\right)\right\|^2\\
    &\leq 2\tilde{\varepsilon}^2 + 2l_j\left( \left(\tilde{\sigma}_{k'}\upj\right)^2+\frac{4L^2\tau_jk'}{b_j}\right) +2q_j\\
    &= 2\tilde{\varepsilon}^2 + 2l_j\left( \frac{4L^2\eta_j^2}{b_j}\sum\limits_{m=1}\limits^{k'}\left\| \bdnu_{m-1}\upj\right\|^2+\frac{4L^2\tau_jk'}{b_j}\right) +2q_j\\
    &\leq 2\tilde{\varepsilon}^2 + 2l_j\left( \frac{4L^2\eta_j^2}{b_j}\sum\limits_{m=1}\limits^{K_j}\left\| \bdnu_{m-1}\upj\right\|^2+\frac{4L^2\tau_jK_j}{b_j}\right) +2q_j\\
    &\leq 2\tilde{\varepsilon}^2 + 2l_j\left( \frac{4L^2\eta_j^2K_j}{b_j}\tilde{\varepsilon}^2+\frac{4L^2\tau_jK_j}{b_j}\right) +2q_j\\
    &=2\tilde{\varepsilon}^2 +2l_j\left( \frac{K_j}{4b_j}\tilde{\varepsilon}^2+\frac{4L^2\tau_jK_j}{b_j}\right) +2q_j\\
    &= 2.5\tilde{\varepsilon}^2 + 8L^2\tau_j +2q_j=2.5\tilde{\varepsilon}^2 +\varepsilon_j\leq \varepsilon^2,
    \end{align*}
where the 2nd step is based on (\ref{eq:importantbound1}) with definitions of $l_j$ and $q_j$ given in Theorem \ref{thm:main1}, the 5th step is based on (\ref{eq:stopcond}), the 6th step is based on the choice of $\eta_j = \frac{1}{4L}$ and the 7th step is based on the choice of $b_j=l_jK_j$. In addition, based on Proposition \ref{prop:nudiff}, the union event $\bigcup\limits_{j=1}\limits^{\infty}\Omega_j$ occurs with probability at least 
$$
1-2\sum\limits_{j=1}\limits^{\infty}\sum\limits_{k=0}\limits^{K_j}\delta'_j \ge 1 - \sum\limits_{j=1}\limits^{\infty}\frac{\delta K_j}{2C_ej^4} \ge 1 - \sum\limits_{j=1}\limits^{\infty}\frac{\delta}{C_ej^2} = 1-\delta.
$$
In one word, it is highly likely to control the norm of gradient at our desired level when the algorithm stops.

The above results can sufficiently explain our choice of stopping rule imposed in Algorithm \ref{algo1}. We can summarize them as the following proposition. 


\begin{proposition}\label{prop:goodstop}
Suppose that Assumptions \ref{assump:smooth} and \ref{assump:technical} are true. Under the parameter setting given in Theorem \ref{thm:main1}, the output of Algorithm \ref{algo1} satisfies
$$
\left\|\nabla f(\hat{\bdx})\right\|^2 \leq \varepsilon^2,
$$
with probability at least $1-\delta$.
\end{proposition}


\subsection{Upper-bounding the Stopping Time}
\label{append:sketch_stopping}


\begin{proposition}[First Stopping Rule]\label{prop:stopyespart1}
Suppose that Assumptions \ref{assump:smooth}, \ref{assump:extent} and \ref{assump:technical} are valid. Let
\begin{align}
    \label{eq_T1_T2}
    & T_1 = \Bigg\lceil \frac{\sqrt{320L(c_1+\Delta_f)}}{\varepsilon} + \frac{320L(c_1+\Delta_f)}{\sqrt{n}\varepsilon^2}\Bigg\rceil, \nonumber \\
    & T_2 = \Bigg\lceil 3\left( \frac{\sqrt{320Lc_2}}{\varepsilon}\log \frac{\sqrt{320Lc_2}}{\varepsilon} + \frac{640Lc_2}{\sqrt{n}\varepsilon^2} \log \frac{320Lc_2}{\varepsilon^2} + 1 \right) \Bigg\rceil,
\end{align}
where
$$
c_1 = \frac{C_eL}{4}+\frac{16\alpha_M^2}{L}\log \frac{192C_e}{\delta},\quad  c_2 = \frac{64\alpha_M^2}{L}.
$$
Under the parameter setting given in Theorem \ref{thm:main1}, on $\Omega$, when $T \ge T_1 \vee T_2$, there exists a $T+1\leq j\leq 2T$ such that
$$
\frac{1}{K_j} \sum\limits_{k=0}\limits^{K_j-1}\left\| \bdnu_k\upj\right\|^2 \leq \tilde{\varepsilon}^2.
$$
\end{proposition}


\begin{proposition}[Second Stopping Rule]\label{prop:stopyespart2}
Let
\begin{align}
    \label{eq_T3_T4}
    T_3 = \Bigg\lceil \frac{2\sqrt{c_3}}{\varepsilon} \Bigg\rceil,\quad T_4 = \Bigg\lceil \frac{6\sqrt{c_4}}{\varepsilon}\log \frac{2\sqrt{c_4}}{\varepsilon} \Bigg\rceil,
\end{align}
where
$$
c_3 = 8L^2+256\alpha_M^2\log\frac{12C_e}{\delta}, \quad c_4 = 1024\alpha_M^2.
$$
Under the parameter setting given in Theorem \ref{thm:main1}, on $\Omega$, when $T\ge T_3 \vee T_4$,
$$
\varepsilon_T \leq \frac{1}{2}\varepsilon^2.
$$
\end{proposition}


\begin{proposition}[Stop Guarantee]\label{prop:stopyes}
Under the parameter setting and assumptions given in Theorem \ref{thm:main1}, on $\Omega$, when $T\ge T_1 \vee T_2 \vee T_3 \vee T_4$, Algorithm \ref{algo1} stops in at most $2T$ outer iterations.
\end{proposition}
\begin{proof}
If Algorithm \ref{algo1} stops in $T$ outer iterations, our conclusion is obviously true. If not, according to Proposition \ref{prop:stopyespart1}, there must exist a $j\in [T+1,2T]$ such that the first stopping rule is met, i.e.
$$
\frac{1}{K_j} \sum\limits_{k=0}\limits^{K_j-1}\left\| \bdnu_k\upj\right\|^2 \leq \tilde{\varepsilon}^2.
$$
According to Proposition \ref{prop:stopyespart2}, the second stopping rule is also met, i.e. $\varepsilon_j \leq \frac{1}{2}\varepsilon^2.$

Consequently, the algorithm stops at the $j$-th outer iteration.
\end{proof}


\section{Technical Lemmas}

\begin{lemma}[Theorem 4 in \textcite{hoeffding1963probability}]
Let $\{\bdepsilon_1,\bdepsilon_2,\ldots,\bdepsilon_n\}$ be a set of fixed vectors in $\RR^d$. $\III,\JJJ\subseteq \{1,2,\ldots,n\}$ are 2 random index sets sampled respectively with replacement and without replacement, with size $|\III|=|\JJJ|=k$. For any continuous and convex function $f:\RR^d\rightarrow \RR$,
$$
\EE f\Big(\sum\limits_{j\in \JJJ}\bdepsilon_j\Big) \leq \EE f\Big(\sum\limits_{i\in \III}\bdepsilon_i\Big).
$$
\label{lemma:hoeffding}
\end{lemma}

\begin{lemma}[Proposition 1.2 in  \textcite{boucheron2013concentration}\footnote{See also Lemma 1.3 in \textcite{bardenet2015concentration}.}]
Let $X$ be real random variable such that $\EE X=0$ and $a\leq X \leq b$ for some $a,b\in \RR$. Then, for all $t\in \RR$,
$$
\log \EE e^{tX} \leq \frac{t^2(b-a)^2}{8}.
$$
\label{lemma:blm}
\end{lemma}


\begin{lemma}[Theorem 3.5 in \textcite{pinelis1994optimum}\footnote{See also Theorem 3 in \textcite{pinelis1992approach} and Proposition 2 in \textcite{fang2018spider}.}] Let $\{\bdepsilon_k\}_{k=1}^{K}\subseteq \RR^d$ be a vector-valued martingale difference sequence with respect to $\FFF_k$, $k=0,1,\ldots,K$, i.e. for $k=1,\ldots,K$, $\EE\left[\bdepsilon_k|\FFF_{k-1}\right]=\textbf{0}$. Assume $\|\bdepsilon_k\|^2 \leq B_k^2$, $k=1,2,\ldots,K$. Then,
$$\PP\left(\Big\| \sum\limits_{k=1}\limits^K \bdepsilon_k\Big\|\ge t\right) \leq 2\text{exp}\Bigg(-\frac{t^2}{2\sum\limits_{k=1}\limits^K B_k^2}\Bigg),
$$
$\forall t\in \RR$.
\label{lemma:pinelis}
\end{lemma}


\begin{proposition}[Norm-Hoeffding, Sampling without Replacement]
Let $\{\bdepsilon_1,\bdepsilon_2,\ldots,\bdepsilon_n\}$ be a set of $n$ fixed vectors in $\RR^d$ such that $\|\bdepsilon_i\|^2 \leq \sigma^2$, $\forall 1\leq i\leq n$, for some $\sigma^2>0$. Let $\JJJ\subseteq \{1,2,\ldots,n\}$ be a random index sets sampled without replacement from $\{1,2,\ldots,n\}$, with size $|\JJJ|=k$. Then,
$$
\PP\Bigg(\Big\|\frac{1}{k}\sum\limits_{j\in \JJJ}\bdepsilon_j - \frac{1}{n}\sum\limits_{j=1}\limits^{n}\bdepsilon_j\Big\| \ge t\Bigg) \leq 3 \text{exp}\Big(-\frac{kt^2}{64\sigma^2}\Big),
$$
$\forall t\in \RR$. In addition,
$$
 \EE\Big\| \frac{1}{k}\sum\limits_{j\in \III}\bdepsilon_j - \frac{1}{n}\sum\limits_{j=1}\limits^{n}\bdepsilon_j\Big\|^2 \leq \frac{16\sigma^2}{k}.
$$
\label{prop:normhoeffding}
\end{proposition}

\begin{proof}
Firstly, we start with developing moment bounds. Let $\III$ be a random index sets sampled with replacement from $\{1,2,\ldots,n\}$, independent of $\JJJ$, with size $|\III|=k$. For any $p\in \ZZ_+$,
\begin{align*}
& \EE\Big\| \frac{1}{k}\sum\limits_{j\in \JJJ}\bdepsilon_j - \frac{1}{n}\sum\limits_{j=1}\limits^{n}\bdepsilon_j\Big\|^p\\
\leq & \EE\Big\| \frac{1}{k}\sum\limits_{j\in \III}\bdepsilon_j - \frac{1}{n}\sum\limits_{j=1}\limits^{n}\bdepsilon_j\Big\|^p\\
= & \bigintss_{0}^{\infty}\PP\Bigg( \Big\| \frac{1}{k}\sum\limits_{j\in \III}\bdepsilon_j - \frac{1}{n}\sum\limits_{j=1}\limits^{n}\bdepsilon_j\Big\|^p \ge r\Bigg) dr\\
= & \bigintss_{0}^{\infty}\PP\Bigg( \Big\| \frac{1}{k}\sum\limits_{j\in \III}\bdepsilon_j - \frac{1}{n}\sum\limits_{j=1}\limits^{n}\bdepsilon_j\Big\| \ge r^{1/p}\Bigg) dr\\
\leq & \bigintss_{0}^{\infty} 2\text{exp}\Bigg(-\frac{kr^{2/p}}{8\sigma^2}\Bigg)dr\\
= & p\cdot \left(\frac{8\sigma^2}{k}\right)^{p/2}\cdot\Gamma\left(\frac{p}{2}\right),
\end{align*}
where the 1st step is based on Lemma \ref{lemma:hoeffding} and the 4th step is based on the fact that $\big\|\bdepsilon_j - \frac{1}{n}\sum\limits_{i=1}\limits^{n}\bdepsilon_i\big\| \leq 2\sigma, \forall j$ and Lemma \ref{lemma:pinelis}.

Then, for any $s>0$,
\begin{align*}
    &\quad \EE\text{exp}\Bigg(s\Big\| \frac{1}{k}\sum\limits_{j\in \JJJ}\bdepsilon_j - \frac{1}{n}\sum\limits_{j=1}\limits^{n}\bdepsilon_j\Big\|\Bigg)\\
    &\leq 1+ \sum\limits_{p=1}\limits^{\infty}\frac{s^p\EE \Big\| \frac{1}{k}\sum\limits_{j\in \JJJ}\bdepsilon_j - \frac{1}{n}\sum\limits_{j=1}\limits^{n}\bdepsilon_j\Big\|^p}{p!}\\
    &\leq 1+ \sum\limits_{p=2}\limits^{\infty}\frac{s^p\EE \Big\| \frac{1}{k}\sum\limits_{j\in \JJJ}\bdepsilon_j - \frac{1}{n}\sum\limits_{j=1}\limits^{n}\bdepsilon_j\Big\|^p}{p!} + s\sqrt{\frac{8\pi\sigma^2}{k}}\\
    &= 1 + s\sqrt{\frac{8\pi\sigma^2}{k}} + \sum\limits_{p=1}\limits^{\infty} \frac{\left(\frac{8\sigma^2 s^2}{k}\right)^p \cdot(2p)\cdot\Gamma(p)}{(2p)!}\\
    &\quad + \sum\limits_{p=1}\limits^{\infty} \frac{\left(\frac{8\sigma^2 s^2}{k}\right)^{\frac{2p+1}{2}} \cdot(2p+1)\cdot\Gamma\left(p+\frac{1}{2}\right)}{(2p+1)!} \\
    &= 1 + s\sqrt{\frac{8\pi\sigma^2}{k}} + 2\sum\limits_{p=1}\limits^{\infty} \frac{\left(\frac{8\sigma^2 s^2}{k}\right)^p \cdot(p!)}{(2p)!}\\
    &\quad + \sqrt{\frac{8s^2\sigma^2}{k}}\sum\limits_{p=1}\limits^{\infty} \frac{\left(\frac{8\sigma^2 s^2}{k}\right)^{p} \cdot\Gamma\left(p+\frac{1}{2}\right)}{(2p)!} \\
    &\leq 1+s\sqrt{\frac{8\pi\sigma^2}{k}} + \left(2+ \sqrt{\frac{8\pi s^2\sigma^2}{k}}\right) \sum\limits_{p=1}\limits^{\infty} \frac{\left(\frac{8\sigma^2 s^2}{k}\right)^p \cdot(p!)}{(2p)!}\\
    &\leq 1+s\sqrt{\frac{8\pi\sigma^2}{k}} + \left(1+ \sqrt{\frac{2\pi s^2\sigma^2}{k}}\right) \sum\limits_{p=1}\limits^{\infty} \frac{\left(\frac{8\sigma^2 s^2}{k}\right)^p }{p!}\\
    &= 1+s\sqrt{\frac{8\pi\sigma^2}{k}} + \left(1+ \sqrt{\frac{2\pi s^2\sigma^2}{k}}\right)\left[\text{exp}\left(\frac{8s^2 \sigma^2}{k}\right)-1\right]\\
    &\leq \sqrt{\frac{8\pi s^2\sigma^2}{k}} + \left(1+ \sqrt{\frac{2\pi s^2\sigma^2}{k}}\right)\text{exp}\left(\frac{8s^2 \sigma^2}{k}\right)\\
    &\leq \text{exp}\left(\frac{8s^2 \sigma^2}{k}\right) + 2\text{exp}\left(\frac{16s^2 \sigma^2}{k}\right)\\
    &\leq 3\text{exp}\left(\frac{16s^2 \sigma^2}{k}\right),
\end{align*}
where the 1st step is based on Taylor's expansion and the second to the last step is based on the fact that $x\leq e^{\frac{x^2}{\pi}}$, $\sqrt{\frac{\pi}{2}}xe^{x^2}\leq e^{2x^2}$, $\forall x\ge 0$.

For any $s>0$,
\begin{equation}
\begin{array}{cl}
& \PP\Bigg( \Big\| \frac{1}{k}\sum\limits_{j\in \JJJ}\bdepsilon_j - \frac{1}{n}\sum\limits_{j=1}\limits^{n}\bdepsilon_j\Big\| \ge t\Bigg)\\

\leq & \EE \text{exp}\Bigg(s \Big\| \frac{1}{k}\sum\limits_{j\in \JJJ}\bdepsilon_j - \frac{1}{n}\sum\limits_{j=1}\limits^{n}\bdepsilon_j\Big\| - ts\Bigg)\\

\leq & 3\text{exp}\left( \frac{16s^2 \sigma^2}{k} - ts \right).\\
\end{array}
\label{eq:chernoff}
\end{equation}
By letting $s=\frac{kt}{32\sigma^2}$ in (\ref{eq:chernoff}),
$$
\PP\Bigg( \Big\| \frac{1}{k}\sum\limits_{j\in \JJJ}\bdepsilon_j - \frac{1}{n}\sum\limits_{j=1}\limits^{n}\bdepsilon_j\Big\| \ge t\Bigg) \leq 3 \text{exp}\left(-\frac{kt^2}{64\sigma^2}\right).
$$
\end{proof}
%

%

\begin{definition}
A random vector $\bdepsilon\in \RR^d$ is $(a,\sigma^2)$-norm-subGaussian (or nSG$(a,\sigma^2)$), if $\exists a,\sigma^2>0$ such that
$$
\PP\left(\|\bdepsilon-\EE \bdepsilon\|\ge t\right) \leq a \cdot \text{exp}\left(-\frac{t^2}{2\sigma^2}\right),
$$
$\forall t\in \RR$.
\end{definition}

\begin{definition}
A sequence of random vectors $\bdepsilon_1,\ldots,\bdepsilon_K\in \RR^d$ is $(a,\{\sigma_k^2\}_{k=1}^{K})$-norm-subGaussian martingale difference sequence adapted to $\FFF_0,\FFF_1,\ldots,\FFF_K$, if $\exists\ a,\sigma^2_1,\ldots,\sigma_K^2>0$ such that for $k=1,2,\ldots,K$, 
$$
\EE\left[\bdepsilon|\FFF_{k-1}\right]=\textbf{0},\ \sigma_k\in \FFF_{k-1},\ \bdepsilon_k\in \FFF_k,
$$
and $\bdepsilon_k|\FFF_{k-1}$ is $(a,\sigma^2_k)$-norm-subGaussian.
\end{definition}

\begin{lemma}[Corollary 8 in \textcite{jin2019short}]
Suppose $\bdepsilon_1,\ldots,\bdepsilon_K\in \RR^d$ is $(a,\{\sigma_k^2\}_{k=1}^{K})$-norm-subGaussian martingale difference sequence adapted to $\FFF_0,\FFF_1,\ldots,\FFF_K$. Then for any fixed $\delta>0$, and $B>b>0$, with probability at least $1-\delta$, either
$$
\sum\limits_{k=1}\limits^K \sigma_i^2\ge B\ 
$$
or,
$$
\Big\| \sum\limits_{k=1}\limits^K \bdepsilon_k\Big\| \leq \frac{a}{2}e^{1/e}\sqrt{\max \left\{ \sum\limits_{k=1}\limits^K \sigma_i^2,b\right\} \left(\log \frac{2d}{\delta} + \log \log \frac{B}{b} \right)}.
$$
\label{lemma:jin}
\end{lemma}










\begin{lemma}
For any $\varepsilon>0,n\in \ZZ_+$,
$$
\frac{1}{T^2\wedge (\sqrt{n}T)}\leq \varepsilon^2,
$$
when $T\ge \lceil \frac{1}{\varepsilon}+\frac{1}{\sqrt{n}\varepsilon^2}\rceil$.
\label{lemma:iteration1}
\end{lemma}

\begin{proof}
\begin{align*}
&\quad\frac{1}{T^2\wedge (\sqrt{n}T)}=\max\Big\{ \frac{1}{T^2},\frac{1}{\sqrt{n}T}\Big\}\\
&\leq \max\Big\{ \frac{1}{T'^2}\big|_{T'=\frac{1}{\varepsilon}},\frac{1}{\sqrt{n}T'}\big|_{T'=\frac{1}{\sqrt{n}\varepsilon^2}}\Big\}\\
&= \varepsilon^2.
\end{align*}
\end{proof}


\begin{lemma}
For any $\varepsilon\in(0,e^{-1}],n\in \ZZ_+$,
$$
\frac{\log T}{T^2\wedge (\sqrt{n}T)}\leq \varepsilon^2,
$$
when $T\ge \Big\lceil 3\left( \frac{1}{\varepsilon}\log \frac{1}{\varepsilon} + \frac{2}{\sqrt{n}\varepsilon^2}\log \frac{1}{\varepsilon^2}+\mathbf{1}\left\{ \frac{1}{\sqrt{n}\varepsilon^2}\log \frac{1}{\varepsilon^2}\leq \frac{e}{6}\right\}\right)\Big\rceil$.
\label{lemma:iteration2}
\end{lemma}

\begin{proof}
Function $h(T)=\frac{\log T}{T^2}$ is monotonically decreasing when $T\ge \sqrt{e}$. Since $T\ge \frac{3}{\varepsilon}\log \frac{1}{\varepsilon}\ge \sqrt{e}$,
\begin{align*}
    &\quad \frac{\log T}{T^2} \leq \frac{\log \left(\frac{3}{\varepsilon}\log \frac{1}{\varepsilon}\right)}{\left(\frac{3}{\varepsilon}\log \frac{1}{\varepsilon}\right)^2} = \frac{\log 3+\log \frac{1}{\varepsilon} + \log \log \frac{1}{\varepsilon}}{9\left( \log \frac{1}{\varepsilon}\right)^2}\varepsilon^2\\
    &\leq \frac{1+\log 3 +2 \log \frac{1}{\varepsilon}}{9\left( \log \frac{1}{\varepsilon}\right)^2}\varepsilon^2 \leq \frac{3+\log 3}{9}\varepsilon^2 \leq \varepsilon^2.
\end{align*}
Define a function $\tilde{h}(T)=\frac{\log T}{T}$. It is monotonically decreasing when $T\ge e$. Thus, if $\frac{6}{\sqrt{n}\varepsilon^2}\log \frac{1}{\varepsilon^2}\ge e$, we know $T\ge e$ and consequently,
\begin{align*}
    &\quad \frac{\log T}{\sqrt{n}T} \leq \frac{\log \left(\frac{6}{\sqrt{n}\varepsilon^2}\log \frac{1}{\varepsilon^2} \right) }{\sqrt{n}\left(\frac{6}{\sqrt{n}\varepsilon^2}\log \frac{1}{\varepsilon^2} \right)}\\
    &=\frac{\log \frac{1}{\sqrt{n}\varepsilon^2} + \log \log \frac{1}{\varepsilon^2} +\log 6}{6 \log \frac{1}{\varepsilon^2}} \varepsilon^2\\
    &\leq \frac{\log \frac{1}{\varepsilon^2} + \log \log \frac{1}{\varepsilon^2} +\log 6}{6 \log \frac{1}{\varepsilon^2}} \varepsilon^2\\
    &\leq \frac{2\log \frac{1}{\varepsilon^2} + 1 +\log 6}{6 \log \frac{1}{\varepsilon^2}} \varepsilon^2\\
    &\leq \frac{3+\log 6}{6}\varepsilon^2\\
    &\leq \varepsilon^2.
\end{align*}
If $\frac{6}{\sqrt{n}\varepsilon^2}\log \frac{1}{\varepsilon^2}\leq e$, $T\ge3$. Hence,
\begin{align*}
&\quad \frac{\log T}{\sqrt{n}T} \leq \frac{\log 3}{\sqrt{n}3} \leq \frac{6}{\sqrt{n}e}\log \frac{1}{\varepsilon^2}\left( \frac{\log 3}{3}\cdot \frac{e}{6}\right)\\
&\leq \frac{6}{\sqrt{n}e}\log \frac{1}{\varepsilon^2} \leq \varepsilon^2.
\end{align*}
Based on the above results,
$$
\frac{\log T}{T^2\wedge (\sqrt{n}T)} \leq \max\left\{ \frac{\log T}{T^2} , \frac{\log T}{ (\sqrt{n}T)}\right\} \leq \varepsilon^2.
$$
\end{proof}


\section{Proofs of Main Theorems}
\label{append:proof}


\begin{proof}[\textbf{Proof of Proposition \ref{prop:stopgua}}]
It is not hard to conclude  that we only need to show
\begin{equation}
    \PP\Bigg( \exists\ j\ge1, \frac{1}{K_j} \sum\limits_{k=0}\limits^{K_j-1}\big\| \bdnu_k\upj\big\|^2 \leq \tilde{\bdepsilon}^2\Bigg) = 1.
    \label{eq:sg1}
\end{equation}
For simplicity, we denote $V_j \triangleq \frac{1}{K_j} \sum\limits_{k=0}\limits^{K_j-1}\big\| \bdnu_k\upj\big\|^2$. To show (\ref{eq:sg1}), we firstly derive the in-expectation bound on $V_j$, which has been covered in works like \textcite{wang2019spiderboost}.

With our basic assumptions, we have
\begin{align*}
    &\quad f\left(\bdx_{k+1}^{(j)}\right)\\
    &=\tilde{f}\Big(\mathrm{Proj}\big( \bdx_k\upj-\eta_j\bdnu_k\upj,\D\big)\Big)\\
    &\leq \tilde{f}\left( \bdx_k\upj-\eta_j\bdnu_k\upj\right)\\
    &\leq \tilde{f}\left( \bdx_k\upj\right) -\Bigl< \nabla \tilde{f}\left( \bdx_k\upj \right),\eta_j\bdnu_k\upj\Bigr> + \frac{L}{2}\eta_j^2\left\|\bdnu_k\upj\right\|^2\\
    &= {f}\left( \bdx_k\upj\right) -\Bigl< \nabla {f}\left( \bdx_k\upj \right),\eta_j\bdnu_k\upj\Bigr> + \frac{L}{2}\eta_j^2\big\|\bdnu_k\upj\big\|^2\\
    &= {f}\left( \bdx_k\upj\right) + \frac{\eta_j}{2}\left\| \bdnu_k\upj - \nabla {f}\big( \bdx_k\upj\big)\right\|^2 - \frac{\eta_j}{2} \left\|  \nabla {f}\big( \bdx_k\upj\big)\right\|^2\\
    &\quad -\frac{\eta_j}{2}\left(1-L\eta_j\right) \big\| \bdnu_k\upj\big\|^2,
\end{align*}
where the 2nd and 3rd step is based on Assumption \ref{assump:extent}. Then, summing the above inequality from $k=0$ to $K_j-1$,
\begin{align}
&\quad f\left( \tilde{\bdx}_j\right) -f\left( \tilde{\bdx}_{j-1}\right) \nonumber\\
& = f\big( \bdx_{K_j}\upj\big) -f\big( \bdx_0\upj\big) \nonumber\\
&\leq \frac{\eta_j}{2}\sum\limits_{k=0}\limits^{K_j-1}\Big\| \bdnu_k\upj - \nabla {f}\big( \bdx_k\upj\big)\Big\|^2 - \frac{\eta_j}{2}\left(1-L\eta_j\right) K_jV_j. \label{eq:sg2}
\end{align}
Then,
\begin{align}
&\quad \EE\big( f\left( \tilde{\bdx}_j\right) -f\left( \tilde{\bdx}_{j-1}\right) \big|\FFF_{j-1}\big) \nonumber\\
&\leq \EE\Bigg(  \frac{\eta_j}{2}\sum\limits_{k=0}\limits^{K_j-1}\left\| \bdnu_k\upj - \nabla {f}\big( \bdx_k\upj\big)\right\|^2 \big|\FFF_{j-1} \Bigg) \nonumber\\
&\quad  -\frac{\eta_j}{2}\left(1-L\eta_j\right) K_j \EE\left( V_j \big| \FFF_{j-1} \right). \label{eq:sg3}
\end{align}

For convenience, we abbreviate $\EE\left( \cdot |\FFF_{j-1}\right)$ as $\EE_{j-1}(\cdot)$. For $k=1,2,\ldots,K_j-1$,
\begin{align}
    &\quad \EE_{j-1}  \left\| \bdnu_k\upj - \nabla {f}\left( \bdx_k\upj\right)\right\|^2  \nonumber\\
    &= \EE_{j-1}  \EE\Bigg( \left\| \bdnu_k\upj - \nabla {f}\left( \bdx_k\upj\right)\right\|^2\big| \FFF_{j,k-1}\Bigg) \nonumber\\
    &= \EE_{j-1}  \EE\Bigg( \Big\|\frac{1}{b_j}\sum\limits_{i\in \III_k\upj}\nabla f_i(\bdx_k\upj)   - \frac{1}{b_j}\sum\limits_{i\in \III_k\upj}\nabla f_i(\bdx_{k-1}\upj)\nonumber\\
    &\quad +\nabla f(\bdx_{k-1}\upj)- \nabla f(\bdx_k\upj) \Big\|^2\Big| \FFF_{j,k-1}\Bigg)\nonumber\\
    &\quad+ \EE_{j-1}\left\|\bdnu_{k-1}\upj - \nabla f(\bdx_{k-1}\upj)\right\|^2\nonumber\\
    &= \EE_{j-1} \frac{1}{b_j^2}\sum\limits_{i\in \III_k\upj} \EE\Bigg(  \Big\|\nabla f_i(\bdx_k\upj)   - \frac{1}{b_j}\sum\limits_{i\in \III_k\upj}\nabla f_i(\bdx_{k-1}\upj)\nonumber\\
    &\quad +\nabla f(\bdx_{k-1}\upj)- \nabla f(\bdx_k\upj) \Big\|^2\Big| \FFF_{j,k-1}\Bigg)\nonumber\\
    &\quad+ \EE_{j-1}\left\|\bdnu_{k-1}\upj - \nabla f(\bdx_{k-1}\upj)\right\|^2\nonumber\\
    &\leq \frac{4L^2}{b_j} \EE_{j-1} \left\| \bdx_k\upj - \bdx_{k-1}\upj\right\|^2 + \EE_{j-1}\left\|\bdnu_{k-1}\upj - \nabla f(\bdx_{k-1}\upj)\right\|^2\nonumber\\
    &\leq \frac{4L^2\eta_j^2}{b_j}\EE_{j-1}\left\| \bdnu_{k-1}\upj\right\|^2 + \EE_{j-1}\left\|\bdnu_{k-1}\upj - \nabla f(\bdx_{k-1}\upj)\right\|^2\nonumber\\
    &\leq \frac{4L^2\eta_j^2}{b_j} \sum\limits_{t=0}\limits^{k-1} \EE_{j-1}\left\| \bdnu_{t}\upj\right\|^2 + \EE_{j-1}\left\|\bdnu_{0}\upj - \nabla f(\bdx_{0}\upj)\right\|^2\nonumber\\
    &\leq \frac{4L^2\eta_j^2 K_j}{b_j}\EE_{j-1}V_j + \EE_{j-1}\left\|\bdnu_{0}\upj - \nabla f(\bdx_{0}\upj)\right\|^2. \label{eq:sg4}
\end{align}
Based on (\ref{eq:sg3}) and (\ref{eq:sg4}),
\begin{align}
&\quad \EE\big(  f\left( \tilde{\bdx}_j\right) -f\left( \tilde{\bdx}_{j-1}\right)  \big) \nonumber\\
&\leq  -\frac{\eta_jK_j}{2}\left(1-L\eta_j - \frac{4L^2\eta_j^2 K_j}{b_j}\right)  \EE V_j \nonumber \\
&\quad + \frac{\eta_jK_j}{2} \EE\left\|\bdnu_{0}\upj - \nabla f(\bdx_{0}\upj)\right\|^2 \nonumber\\
&\leq -\frac{K_j}{16L} \EE V_j + \frac{K_j}{8L}\EE\left\|\bdnu_{0}\upj - \nabla f(\bdx_{0}\upj)\right\|^2, \label{eq:sg5}
\end{align}
where the second step is based on the choice of $\eta_j \equiv \frac{1}{4L}$ and $b_j \ge K_j$, $j\ge 1$.
Let us define $J_0 = \min\{j:B_j = n\}$. Then, for $j\ge J_0$, based on (\ref{eq:sg5}),
$$
\frac{1}{16L}\EE V_j \leq \frac{K_j}{16L} \EE V_j \leq \EE\left(  f\left( \tilde{\bdx}_{j-1}\right) -f\left( \tilde{\bdx}_{j}\right)  \right).
$$
Then for any $m\in \ZZ_+$,
\begin{align*}
&\quad \PP \left( V_j > \tilde{\bdepsilon}^2,\ j\ge1\right) \\
&\leq \PP \left( V_{J_0} + V_{J_0+1}+\ldots +V_{J_0+m} > (m+1)\tilde{\bdepsilon}^2\right)\\
&\leq \frac{\EE\left(  V_{J_0} + V_{J_0+1}+\ldots +V_{J_0+m}\right) }{(m+1)\tilde{\bdepsilon}^2}\\
&\leq \frac{16L}{(m+1)\tilde{\bdepsilon}^2}\sum\limits_{j=J_0}\limits^{J_0+m} \EE\left(  f\left( \tilde{\bdx}_{j-1}\right) -f\left( \tilde{\bdx}_{j}\right)  \right)\\
&\leq \frac{16L\Delta_f}{(m+1)\tilde{\bdepsilon}^2}.
\end{align*}
Since $m$ can be arbitrarily large, we know
$$
\PP \left( V_j > \tilde{\bdepsilon}^2,\ j\ge1\right)=0,
$$
which can directly lead to (\ref{eq:sg1}).
\end{proof}


\begin{proof}[\textbf{Proof of Theorem \ref{thm:martingaleah}}]
In this proof, for simplicity, we denote $\EE\left[\cdot|\FFF_k\right]$ by $\EE_k\left[\cdot\right]$. Let $\bds_k=\sum\limits_{i=1}\limits^{k}\bdz_i,k\ge 1$. For a $1\leq k \leq K$, consider
$$
f_k(t)=\EE_{k-1}\left[cosh(\lambda\|\bds_{k-1}+t\bdz_k\|)\right],\ \lambda>0,t>0.
$$
Then,
$$
f'_k(t) = \frac{1}{2}\EE_{k-1}\left[ \frac{\lambda\langle \bdz_k,\bds_{k-1}+t\bdz_k\rangle}{\|\bds_{k-1}+t\bdz_k\|}\left( e^{\lambda \|\bds_{k-1}+t\bdz_k\|}- e^{-\lambda \|\bds_{k-1}+t\bdz_k\|}\right) \right],
$$
and consequently,
\begin{align*}
f'_k(0) &= \frac{1}{2}\EE_{k-1}\left[ \frac{\lambda\langle \bdz_k,\bds_{k-1}\rangle}{\|\bds_{k-1}\|}\left( e^{\lambda \|\bds_{k-1}\|}- e^{-\lambda \|\bds_{k-1}\|}\right) \right]\\
&= 0.
\end{align*}
Next,
\begin{align*}
&\quad f''_k(t)\\
&= \frac{1}{2}\EE_{k-1}\Biggl[ \left(\frac{\lambda^2\langle \bdz_k,\bds_{k-1}+t\bdz_k\rangle^2}{\|\bds_{k-1}+t\bdz_k\|^2} + \frac{\lambda\|\bdz_k\|^2}{\|\bds_{k-1}+t\bdz_k\|}\right) e^{\lambda \|\bds_{k-1}+t\bdz_k\|}\\
&\quad\quad +  \left(\frac{\lambda^2\langle \bdz_k,\bds_{k-1}+t\bdz_k\rangle^2}{\|\bds_{k-1}+t\bdz_k\|^2} - \frac{\lambda\|\bdz_k\|^2}{\|\bds_{k-1}+t\bdz_k\|}\right) e^{-\lambda \|\bds_{k-1}+t\bdz_k\|}\Biggr] \\
&= \EE_{k-1}\Biggl[\frac{\lambda^2\langle \bdz_k,\bds_{k-1}+t\bdz_k\rangle^2}{\|\bds_{k-1}+t\bdz_k\|^2} cosh(\lambda\|\bds_{k-1}+t\bdz_k\|)\\
&\quad \quad + \frac{\lambda^2 \|\bdz_k\|^2}{\lambda \|\bds_{k-1}+t\bdz_k\|}sinh(\lambda\|\bds_{k-1}+t\bdz_k\|)\Biggr]\\
&\leq \EE_{k-1}\Biggl[\left(\frac{\lambda^2\langle \bdz_k,\bds_{k-1}+t\bdz_k\rangle^2}{\|\bds_{k-1}+t\bdz_k\|^2} +\lambda^2 \|\bdz_k\|^2\right) cosh(\lambda\|\bds_{k-1}+t\bdz_k\|)\Biggr]\\
&\leq 2\lambda^2 \EE_{k-1}\left[ \|\bdz_k\|^2 cosh(\lambda\|\bds_{k-1}+t\bdz_k\|)\right]\\
&\leq 2\lambda^2r_k^2 \EE_{k-1}\left[ cosh(\lambda\|\bds_{k-1}+t\bdz_k\|)\right]\\
&= 2\lambda^2r_k^2f_k(t),
\end{align*}
where the first inequality is based on the fact that if $y>0$, $\frac{sinh(y)}{y}\leq cosh(y)$. 

According to Lemma 3 in \textcite{pinelis1992approach},
$$
f_k(t) \leq f_k(0)\text{exp}\left( \lambda^2r_k^2t^2\right) = cosh(\lambda\|\bds_{k-1}\|)\text{exp}\left( \lambda^2r_k^2t^2\right).
$$
Thus,
\begin{align}
&\quad \EE_{k-1}\left[cosh(\lambda\|\bds_{k}\|)\right]\nonumber\\
&= f_k(1) \leq cosh(\lambda\|\bds_{k-1}\|)\text{exp}\left( \lambda^2r_k^2t^2\right).\label{aheq1}
\end{align}
Now, let
$$
G_k=cosh(\lambda\|\bds_k\|)\text{exp}\Big( -\lambda^2\sum\limits_{i=1}\limits^{k}r_i^2\Big),\ k=1,2,\ldots,K.
$$
We can easily know that for $k=1,2,\ldots,K$, $G_k$ is measurable with respect to $\FFF_k$. According to (\ref{aheq1}),
\begin{align*}
\EE_{k-1}G_k &= \text{exp}\Big( -\lambda^2\sum\limits_{i=1}\limits^{k}r_i^2\Big) \EE_{k-1}\left[cosh(\lambda\|\bds_{k}\|)\right]\\
&\leq cosh(\lambda\|\bds_{k-1}\|)\text{exp}\Big( -\lambda^2\sum\limits_{i=1}\limits^{k-1}r_i^2\Big)\\
&=G_{k-1},
\end{align*}
which implies that $\{G_k\}_{k=1}^{K}$ is a non-negative super-martingale adapted to $\FFF_0,\FFF_1,\ldots,\FFF_K$. 

For any constant $m>0$, if we define stopping time $T_m = \inf\left\{ t:\|\bds_t\|\ge \lambda \sum\limits_{i=1}\limits^t r_i^2 +m \right\}$, we immediately know that $G_{T_m\wedge k}$, $k\ge0$, is a supermartingale and 
\begin{align*}
&\quad \PP\left( \exists 1\leq t\leq k, \|\bds_t\|\ge \lambda \sum\limits_{i=1}\limits^t r_i^2 +m\right)\\
&= \PP \left( \|\bds_{T_m}\|\ge \lambda \sum\limits_{i=1}\limits^{T_m} r_i^2 +m, 1\leq T_m\leq k\right)\\
&= \PP \left( \|\bds_{T_m\wedge k}\|\ge \lambda \sum\limits_{i=1}\limits^{T_m \wedge k} r_i^2 +m, 1\leq T_m\leq k\right)\\
&\leq \PP\left( G_{T_m \wedge k} \ge \text{exp}\Big( -\lambda^2\sum\limits_{i=1}\limits^{{T_m \wedge k}}r_i^2\Big) cosh\Big( \lambda^2 \sum\limits_{i=1}\limits^{T_m \wedge k} r_i^2 +m\lambda\Big) \right)\\
&\leq \PP\Bigg( G_{T_m \wedge k} \ge \frac{1}{2}\text{exp}\Big( -\lambda^2\sum\limits_{i=1}\limits^{{T_m \wedge k}}r_i^2+ \big( \lambda^2\sum\limits_{i=1}\limits^{{T_m \wedge k}}r_i^2+m\lambda\big)\Big) \Bigg)\\
&=\PP\left( 2G_{T_m \wedge k}\ge e^{\lambda m}\right)\\
&\leq \frac{2\EE G_{T_m \wedge k}}{e^{\lambda m}}\\
&\leq 2e^{-\lambda m}\EE G_0\\
&= 2e^{-\lambda m},
\end{align*}
where the 2nd step is based on the fact that $cosh(y)\ge \frac{1}{2}e^y, \forall y\in \RR$, the 4th step is by Chebyshev's inequality and the 5th step is based on the supermartingale property.

Therefore, if we let $\lambda m =\log \frac{2}{\delta}$,
$$
\PP\left( \exists 1\leq t\leq k, \|\bds_t\|\ge \lambda \sum\limits_{i=1}\limits^t r_i^2 + \frac{1}{\lambda}\log \frac{2}{\delta}\right) \leq \delta.
$$
Since $k$ can be up to $K$,
$$
\PP\left( \exists 1\leq t\leq K, \|\bds_t\|\ge \lambda \sum\limits_{i=1}\limits^t r_i^2 + \frac{1}{\lambda}\log \frac{2}{\delta}\right) \leq \delta.
$$
The final conclusion can be obtained immediately by following similar steps given in the proof of Corollary 8 from \textcite{jin2019short}.
\end{proof}



\begin{proof}[\textbf{Proof of Proposition \ref{prop:nudiff}}]
Recall that
$$
\bdepsilon_0\upj= \bdnu_0\upj-\nabla f\left(\bdx_0\upj\right) = \frac{1}{B_j}\sum\limits_{i\in \III_j} \nabla f_i\left(\bdx_0\upj\right)-\nabla f\left(\bdx_0\upj\right),
$$
where $\III_j$ is sampled without replacement. Since $\Big\|\nabla f_i\big(\bdx_0\upj\big)\Big\| \leq \alpha_M$, $i=1,2,\ldots,n$, based on Proposition \ref{prop:normhoeffding},
\begin{equation}
\PP\left(\|\bdepsilon_0\upj\| \ge t | \FFF_{j,-1}\right) \leq 3\text{exp}\left(-\frac{B_j t^2}{64 \alpha_M^2}\right) \mathbf{1}\left\{B_j < n\right\}.
\label{eq:prop1_1}
\end{equation}

Next, if we suppose $\III_m\upj = \left\{ i_{m,1}\upj,i_{m,2}\upj,\ldots,i_{m,b_j}\upj\right\}$, where $i_{m,t_1}\upj \ne i_{m,t_2}\upj$ for any $1\leq t_1 < t_2 \leq b_j$, we have
\begin{align*}
&\quad\bdepsilon_m\upj\\
&= \frac{1}{b_j}\sum\limits_{i\in \III_m\upj} \Biggl[ \nabla f_i(\bdx_m\upj) - \nabla f(\bdx_m\upj) +\nabla f(\bdx_{m-1}\upj)- \nabla f_i(\bdx_{m-1}\upj)\Biggr]\\
&= \sum\limits_{r=1}\limits^{b_j} \frac{1}{b_j}\Biggl[ \nabla f_{i_{m,r}\upj}(\bdx_m\upj) - \nabla f(\bdx_m\upj) +\nabla f(\bdx_{m-1}\upj)- \nabla f_{i_{m,r}\upj}(\bdx_{m-1}\upj)\Biggr]\\
&\triangleq \sum\limits_{r=1}\limits^{b_j} \bdrho_{(m-1)b_j+r}\upj.
\end{align*}

Let 
$$
\tilde{\FFF}_0\upj = \FFF_{j,0}
$$
and 
$$
\tilde{\FFF}_{a_1b_j+a_2}\upj=\sigma\left( \tilde{\FFF}_{a_1b_j+a_2-1}\upj \bigcup \sigma\big( i_{a_1+1,a_2}\upj\big)\right)
$$
for $a_1=0,1,2,\ldots$ and $a_2 = 1,2,\ldots,b_j$. Then, we can see that $\big\{\bdrho_s\upj\big\}_{s=1}^{kb_j}$ is a martingale difference sequence adapted to $\big\{\tilde{\FFF}_s\upj\big\}_{s=0}^{kb_j}$.

Notice that for $m=1,2,\ldots,k$ and $r=1,2,\ldots,b_j$,
\begin{align*}
&\quad \big\| \bdrho_{(m-1)b_j+r}\upj\big\|\\
&= \left\|\frac{1}{b_j}\big[ \nabla f_{i_{m,r}\upj}(\bdx_m\upj) - \nabla f(\bdx_m\upj) +\nabla f(\bdx_{m-1}\upj)- \nabla f_{i_{m,r}\upj}(\bdx_{m-1}\upj)\big] \right\|\\
&\leq \frac{2L}{b_j}\left\| \bdx_m\upj - \bdx_{m-1}\upj\right\|.
\end{align*}

Therefore, based on Theorem \ref{thm:martingaleah}, for any fixed $\delta'>0$, $B>b>0$, with probability at least $1-\delta'$, either
\begin{align*}
\left( \sigma_k\upj\right)^2 &\triangleq \sum\limits_{m=1}\limits^k \sum\limits_{r=1}\limits^{b_j}\left( \frac{2L}{bj} \left\| \bdx_m\upj - \bdx_{m-1}\upj\right\| \right)^2\\
&= \frac{4L^2}{b_j}\sum\limits_{m=1}\limits^k \left\| \bdx_m\upj - \bdx_{m-1}\upj\right\|^2\\
&\ge B,
\end{align*}
or
\begin{align*}
&\quad\left\| \bdnu_k\upj -\nabla f\left( \bdx_k\upj \right) -\bdepsilon_0\upj \right\|^2\\
&=\Big\| \sum\limits_{s=1}\limits^{kb_j}\bdrho_s\upj\Big\|^2\\
&\leq 9\max \left\{ \big( \sigma_k\upj\big)^2 ,b\right\} \left( \log \frac{2}{\delta'}+\log\log \frac{B}{b}\right).
\end{align*}

Under the compact constraint,
$$
\big( \sigma_k\upj\big)^2 \leq \frac{4L^2d_1^2k}{b_j}.
$$
Thus, if we let $B=\frac{8L^2d_1^2k}{b_j}$ and $b=\frac{4L^2\tau k}{b_j}$ for some $\tau\in (0,1)$, it would be of probability 0 to have $\big( \sigma_k\upj\big)^2 \ge B$. Thus, with probability at least $1-\delta'$,
\begin{align}
&\quad\left\| \bdnu_k\upj -\nabla f\left( \bdx_k\upj \right) -\bdepsilon_0\upj \right\|^2\nonumber\\
&\leq 9\left( \big( \sigma_k\upj\big)^2 + \frac{4L^2\tau k}{b_j} \right) \left( \log \frac{2}{\delta'} + \log \log \frac{2d_1^2}{\tau}\right)\label{eq:prop1_2}\\
&\leq 9\left( \big( \tilde{\sigma}_k\upj\big)^2 + \frac{4L^2\tau k}{b_j} \right) \left( \log \frac{2}{\delta'} + \log \log \frac{2d_1^2}{\tau}\right).\nonumber
\end{align}
According to (\ref{eq:prop1_1}), with probability at least $1-\delta'$,
\begin{equation}
\big\| \bdepsilon_0\upj\big\|^2 \leq \frac{64\alpha_M^2}{B_j}\log\frac{3}{\delta'}\mathbf{1}\left\{B_j<n\right\}.
\label{eq:prop1_3}
\end{equation}
Thus, combining (\ref{eq:prop1_2}) and (\ref{eq:prop1_3}), with probability at least $1-2\delta'$,
\begin{align*}
&\quad \left\| \bdnu_k\upj -\nabla f\left( \bdx_k\upj\right)\right\|^2\\
&\leq 18\left( \big( \tilde{\sigma}_k\upj\big)^2 + \frac{4L^2\tau k}{b_j} \right) \left( \log \frac{2}{\delta'} + \log \log \frac{2d_1^2}{\tau}\right)\\
&\quad+ \frac{128\alpha_M^2}{B_j}\log\frac{3}{\delta'}\mathbf{1}\left\{B_j<n\right\}.
\end{align*}
\end{proof}

\begin{proposition}[Inner Loop Analysis]
Given Assumptions \ref{assump:smooth}, \ref{assump:extent} and \ref{assump:technical}, under the parameter setting given in Theorem \ref{thm:main1}, let $\Omega=\bigcup\limits_{j=1}\limits^{\infty} \Omega_j$, where the definition of $\Omega_j$ is given in (\ref{eq:omegaj}). On $\Omega$,
\begin{align*}
&\quad f\left( \bdx_{K_j}\upj\right) -f\left( \bdx_0\upj\right)\\
&\leq -\frac{1}{16L}\sum\limits_{k=0}\limits^{K_j-1} \left\| \bdnu_k\upj\right\|^2+ \frac{L^2\eta_j \tau_j l_j K_j^2}{b_j}+\frac{\eta_jK_jq_j}{2},
\end{align*}
for all $j\in \ZZ_+$. Such event $\Omega$ occurs with probability at least $1-\delta$.
\label{prop:innerloop}
\end{proposition}

\begin{proof}[\textbf{Proof of Proposition \ref{prop:innerloop}}]
Firstly, as what we have shown in section \ref{sec:theory}, $\bigcup\limits_{j=0}\limits^{\infty}\Omega_j$ occurs with probability at least $1-\delta$. 

Based on (\ref{eq:sg2}), on $\bigcup\limits_{j=0}\limits^{\infty}\Omega_j$,
\begin{align*}
    &\quad f\left( \bdx_{K_j}\upj\right) -f\left( \bdx_0\upj\right)\\
    &\leq \frac{\eta_j}{2}\sum\limits_{k=0}\limits^{K_j-1}\left[l_j\left( \left( \tilde{\sigma}_k\upj\right)^2+\frac{4L^2\tau_jk}{b_j}\right) +q_j\right]\\
    &\quad - \frac{\eta_j}{2}(1-L\eta_j)\sum\limits_{k=0}\limits^{K_j-1} \left\| \bdnu_k\upj\right\|^2\\
    &\leq \frac{\eta_j}{2}\sum\limits_{k=0}\limits^{K_j-1}\left( \frac{4L^2\eta_j^2l_j}{b_j}\sum\limits_{m=1}\limits^k\left\| \bdnu_{m-1}\upj\right\|^2\right) + \frac{2L^2\eta_j\tau_j l_j}{b_j}\frac{K_j^2}{2}\\
    &\quad +\frac{\eta_jK_jq_j}{2}- \frac{\eta_j}{2}(1-L\eta_j)\sum\limits_{k=0}\limits^{K_j-1} \left\| \bdnu_k\upj\right\|^2\\
    &\leq \frac{2L^2\eta_j^3l_jK_j}{b_j}\sum\limits_{k=0}\limits^{K_j-1} \left\| \bdnu_k\upj\right\|^2 + \frac{L^2\eta_j \tau_j l_j K_j^2}{b_j}+\frac{\eta_jK_jq_j}{2}\\
    &\quad- \frac{\eta_j}{2}(1-L\eta_j)\sum\limits_{k=0}\limits^{K_j-1} \left\| \bdnu_k\upj\right\|^2\\
    &=-\frac{\eta_j}{2}\left( 1-L\eta_j - \frac{4L^2\eta_j^2l_jK_j}{b_j}\right) \sum\limits_{k=0}\limits^{K_j-1} \left\| \bdnu_k\upj\right\|^2\\
    &\quad+ \frac{L^2\eta_j \tau_j l_j K_j^2}{b_j}+\frac{\eta_jK_jq_j}{2}\\
    &= -\frac{1}{8L}\left( 1-\frac{1}{4}-\frac{1}{4}\right)\sum\limits_{k=0}\limits^{K_j-1} \left\| \bdnu_k\upj\right\|^2+ \frac{L^2\eta_j \tau_j l_j K_j^2}{b_j}+\frac{\eta_jK_jq_j}{2}\\
    &=-\frac{1}{16L}\sum\limits_{k=0}\limits^{K_j-1} \left\| \bdnu_k\upj\right\|^2+ \frac{L^2\eta_j \tau_j l_j K_j^2}{b_j}+\frac{\eta_jK_jq_j}{2},
\end{align*}
where the 5th step is based on our choices of $\eta_j=\frac{1}{4L}$ and $b_j=l_jK_j$, $j=1,2,\ldots$.
\end{proof}


\begin{proof}[\textbf{Proof of Proposition \ref{prop:stopyespart1}}]

Firstly,
\begin{align}
    &\quad -\Delta_f \leq f\left( \tilde{\bdx}_{{2T}}\right) - f\left( \tilde{\bdx}_{{T}}\right) = f\left( \bdx_{K_{2T}}^{(2T)}\right) -f\left( \bdx_0^{(T+1)}\right) \nonumber\\
    &\leq \sum\limits_{j=T+1}\limits^{2T}\left[\frac{-1}{16L}\sum\limits_{k=0}\limits^{K_j-1} \left\| \bdnu_k\upj\right\|^2+ \frac{L^2\eta_j \tau_j l_j K_j^2}{b_j}+\frac{\eta_jK_jq_j}{2} \right]\label{eq:fvaluebound}\\
    &= \sum\limits_{j=T+1}\limits^{2T}\left[\frac{L^2\eta_j \tau_j l_j K_j^2}{b_j}+\frac{\eta_jK_jq_j}{2} \right] - \frac{1}{16L}\sum\limits_{j=T+1}\limits^{2T}\sum\limits_{k=0}\limits^{K_j-1} \left\| \bdnu_k\upj\right\|^2,\nonumber
\end{align}
where the 3rd step is based on Proposition \ref{prop:innerloop}.

For simplifying notations, we denote
$$
A_T \triangleq \sum\limits_{j=T+1}\limits^{2T}\left[\frac{L^2\eta_j\tau_jl_jK_j^2}{b_j} + \frac{\eta_jK_jq_j}{2} \right].
$$
Then,
\begin{align}
    &\quad A_T\nonumber\\
    &= \sum\limits_{j=T+1}\limits^{2T} \left(\frac{L\tau_jK_j}{4}+\frac{K_jq_j}{8L}\right)\nonumber\\
    &= \sum\limits_{j=T+1}\limits^{2T} \left( \frac{L\sqrt{j^2\wedge n}}{4j^3}+\frac{16\alpha_M^2}{L\sqrt{j^2\wedge n}}\log\frac{12C_ej^4}{\delta} \mathbf{1}\left\{j^2<n\right\}\right)\nonumber\\
    &\leq \sum\limits_{j=T+1}\limits^{2T} \left( \frac{L}{4j^2} +\frac{16\alpha_M^2}{L\sqrt{j^2\wedge n}}\log\frac{12C_ej^4}{\delta} \mathbf{1}\left\{j^2<n\right\}\right)\nonumber\\
    &\leq \sum\limits_{j=T+1}\limits^{2T} \left( \frac{L}{4j^2} +\frac{16\alpha_M^2}{Lj}\log\frac{12C_ej^4}{\delta} \right)\label{eq:longbound1}\\
    &\leq \frac{C_eL}{4} + \sum\limits_{j=T+1}\limits^{2T}\frac{16\alpha_M^2}{Lj}\log\frac{12C_ej^4}{\delta}\nonumber\\
    &\leq \frac{C_eL}{4} + \sum\limits_{j=T+1}\limits^{2T}\frac{16\alpha_M^2}{LT}\log\frac{12C_e(2T)^4}{\delta}\nonumber\\
    &= \frac{C_eL}{4} + \frac{16\alpha_M^2}{L}\log \frac{192C_eT^4}{\delta}\nonumber\\
    &= \frac{C_eL}{4} + \frac{16\alpha_M^2}{L}\log \frac{192C_e}{\delta} + \frac{64\alpha_M^2}{L}\log T\nonumber,
\end{align}
where the 1st step is based on the choices that $\eta_j=\frac{1}{4L}$ and $b_j=l_jK_j$, the second step is bases on the choices of $K_j=\sqrt{B_j}=\sqrt{j^2\wedge n}$, $\delta'_j=\frac{\delta}{4C_ej^4}$. According to Lemma \ref{lemma:iteration1}, as $T\ge T_1$,
\begin{equation}
\frac{1}{T^2\wedge (\sqrt{n}T)} \leq \frac{\varepsilon^2}{320 L (c_1+\Delta_f)}.
\label{eq:stopbound1}
\end{equation}
According to Lemma \ref{lemma:iteration2}, as $T\ge T_2$,
\begin{equation}
\frac{\log T}{T^2\wedge (\sqrt{n}T)} \leq \frac{\varepsilon^2}{320 L c_2}.
\label{eq:stopbound2}
\end{equation}
If we suppose to the contrary that
$$
\frac{1}{K_j} \sum\limits_{k=0}\limits^{K_j-1}\left\| \bdnu_k\upj\right\|^2 > \tilde{\varepsilon}^2
$$
holds for all $T+1\leq j \leq 2T$, then we have
{\small
\begin{align*}
&\quad \frac{1}{16L}\sum\limits_{j=T+1}\limits^{2T} \sum\limits_{k=0}\limits^{K_j-1} \left\| \bdnu_k\upj\right\|^2 \ge \frac{\tilde{\varepsilon}^2}{16L} \sum\limits_{j=T+1}\limits^{2T}K_j\nonumber  \ge \frac{\tilde{\varepsilon}^2}{16L} \sum\limits_{j=T+1}\limits^{2T} \left( T\wedge \sqrt{n} \right) = \frac{\tilde{\varepsilon}^2}{16L}T^2 \wedge (\sqrt{n}T).
\end{align*}
}
By (\ref{eq:longbound1}), (\ref{eq:stopbound1}), (\ref{eq:stopbound2}) and the above results,
\begin{align*}
&\quad\frac{80L}{T^2\wedge(\sqrt{n}T)} \left\{ \Delta_f + A_T -\frac{1}{16L} \sum\limits_{j=T+1}\limits^{2T} \sum\limits_{k=0}\limits^{K_j-1} \left\| \bdnu_k\upj\right\|^2 \right\}\\
&\leq \frac{80L}{T^2\wedge(\sqrt{n}T)} (\Delta_f + A_T) - 5\tilde{\varepsilon}^2\\
&=  \frac{80L}{T^2\wedge(\sqrt{n}T)} (\Delta_f + A_T) -\varepsilon^2\\
&\leq \frac{80L}{T^2\wedge(\sqrt{n}T)}(\Delta_f + c_1) + \frac{80Lc_2\log T}{T^2\wedge(\sqrt{n}T)} - \varepsilon^2\\
&\leq \frac{\varepsilon^2}{4} + \frac{\varepsilon^2}{4} - \varepsilon^2\\
&= -\frac{\varepsilon^2}{2},
\end{align*}
which contradicts (\ref{eq:fvaluebound}). 
\end{proof}


\begin{proof}[\textbf{Proof of Proposition \ref{prop:stopyespart2}}]
\begin{align}
&\quad \varepsilon_T\nonumber\\
&= 8L^2\tau_T+2q_T\nonumber\\
&= \frac{8L^2}{T^3} + \frac{256\alpha_M^2}{B_T}\log\frac{3}{\delta'_T}\mathbf{1}\left\{B_T<n\right\}\nonumber\\
&\leq \frac{8L^2}{T^3} + \frac{256\alpha_M^2}{T^2}\log\frac{3}{\delta'_T}\label{eq:longstopbound2}\\
&= \frac{8L^2}{T^3} + \frac{256\alpha_M^2}{T^2}\log\frac{12C_eT^4}{\delta}\nonumber\\
&\leq \left( 8L^2 + 256\alpha_M^2\log \frac{12C_e}{\delta}\right) \frac{1}{T^2} + \frac{1024\alpha_M^2}{T^2}\log T\nonumber\\
&=c_3\frac{1}{T^2}+c_4\frac{\log T}{T^2},
\end{align}
where the 2nd step is based on our choice of $\tau_T=\frac{1}{T^3}$ and the 4th step is based on the choice of $\delta'_T=\frac{\delta}{4C_eT^4}$. According to Lemma \ref{lemma:iteration1}, where we can simply let $n=\infty$, as $T\ge T_3$,
\begin{equation}
\frac{1}{T^2} \leq \frac{\varepsilon^2}{4c_3}.
\label{eq:stopbound3}
\end{equation}
Similarly, according to Lemma \ref{lemma:iteration2} and Assumption \ref{assump:technical}, as $T\ge T_4$,
\begin{equation}
\frac{\log T}{T^2} \leq \frac{\varepsilon^2}{4 c_4}.
\label{eq:stopbound4}
\end{equation}
Combining (\ref{eq:longstopbound2}), (\ref{eq:stopbound3}) and (\ref{eq:stopbound4}),
$$
\varepsilon_T \leq \frac{\varepsilon^2}{2}.
$$
\end{proof}


\begin{proof}[\textbf{Proof of Corollary \ref{cor:setting1_append}}]
This part follows a similar way as the complexity analysis in \textcite{horvath2020adaptivity}. It is easy to know that if Algorithm \ref{algo1} stops in $T$ outer iterations, the first order computational complexity is
$$
\tilde{\OOO}_{L,\Delta_f,\alpha_M} \left( T^3\wedge (nT)\right).
$$
Thus, it is sufficient to show 
$$
T_i^3\wedge (nT_i) = \tilde{\OOO}_{L,\Delta_f,\alpha_M}\left(\frac{1}{\varepsilon^3}\wedge \frac{\sqrt{n}}{\varepsilon^2}\right),\ i=1,2,3,4.
$$

\noindent \textbf{$\bullet\ T_1^3\wedge (nT_1)$}

For simplicity, we let $\tilde{c}_1 = \sqrt{320L(c_1+\Delta_f)}$ and consequently $T_1 = \Big\lceil \frac{\tilde{c}_1}{\varepsilon} + \frac{\tilde{c}_1^2}{\sqrt{n}\varepsilon^2}\Big\rceil$.

When $\sqrt{n}\varepsilon \leq \tilde{c}_1$, $\frac{\tilde{c}_1}{\varepsilon} \leq \frac{\tilde{c}_1^2}{\sqrt{n}\varepsilon^2}$ and consequently $T=\OOO \left( \frac{\tilde{c}_1^2}{\sqrt{n}\varepsilon^2}\right)$. Hence,
$$
T_1^3\wedge (nT_1) = \OOO(nT_1) = \OOO \left( \frac{\sqrt{n}\tilde{c}_1^2}{\varepsilon^2}\right)= \OOO \left( \frac{\sqrt{n}\tilde{c}_1^2}{\varepsilon^2} \wedge \frac{\tilde{c}_1^3}{\varepsilon^3}\right),
$$
where the last step is due to $\sqrt{n} \leq \frac{\tilde{c}_1}{\varepsilon}$.

When $\sqrt{n}\varepsilon \ge \tilde{c}_1$, $\frac{\tilde{c}_1}{\varepsilon} \ge \frac{\tilde{c}_1^2}{\sqrt{n}\varepsilon^2}$ and consequently $T=\OOO \left( \frac{\tilde{c}_1}{\varepsilon}\right)$. Hence,
$$
T_1^3\wedge (nT_1) = \OOO(T_1^3) = \OOO \left( \frac{\tilde{c}_1^3}{\varepsilon^3}\right)= \OOO \left( \frac{\sqrt{n}\tilde{c}_1^2}{\varepsilon^2} \wedge \frac{\tilde{c}_1^3}{\varepsilon^3}\right),
$$
where the last step is due to $\tilde{c}_1 \leq \sqrt{n}\varepsilon$.

To sum up,
$$
T_1^3\wedge (nT_1) = \OOO \left( \frac{\sqrt{n}\tilde{c}_1^2}{\varepsilon^2} \wedge \frac{\tilde{c}_1^3}{\varepsilon^3}\right) = \tilde{\OOO}_{L,\Delta_f,\alpha_M}\left(\frac{1}{\varepsilon^3}\wedge \frac{\sqrt{n}}{\varepsilon^2}\right).
$$
\noindent \textbf{$\bullet\ T_2^3\wedge (nT_2)$}

Secondly, if we let $\tilde{c}_2=\sqrt{320Lc_2}$, we have $\tilde{c}_2\ge 4$ based on Assumption \ref{assump:technical}. As a result, $T_2 =\Theta\left( \frac{3\tilde{c}_2}{\varepsilon} \log \frac{3\tilde{c}_2}{\varepsilon} + \frac{2\tilde{c}_2^2}{\sqrt{n}\varepsilon^2} \log \frac{\tilde{c}_2^2}{\varepsilon^2}\right) $. Therefore, it is equivalent to study $\bar{T}_2^3\wedge (n\bar{T}_2)$ where $\bar{T}_2=\frac{3\tilde{c}_2}{\varepsilon} \log \frac{3\tilde{c}_2}{\varepsilon} + \frac{2\tilde{c}_2^2}{\sqrt{n}\varepsilon^2} \log \frac{\tilde{c}_2^2}{\varepsilon^2}$.

When $\tilde{c}_2 \ge 1.5\sqrt{n}\varepsilon$,
$$
\frac{3\tilde{c}_2}{\varepsilon} \log \frac{3\tilde{c}_2}{\varepsilon} \leq \frac{3\tilde{c}_2}{\varepsilon} \log \frac{\tilde{c}_2^2}{\varepsilon^2}  \leq \frac{2\tilde{c}_2^2}{\sqrt{n}\varepsilon^2} \log \frac{\tilde{c}_2^2}{\varepsilon^2}.
$$
Thus, $\bar{T}_2 = \OOO \left( \frac{2\tilde{c}_2^2}{\sqrt{n}\varepsilon^2} \log \frac{\tilde{c}_2^2}{\varepsilon^2} \right)$. Then,
\begin{align*}
&\quad\bar{T}_2^3 \wedge(n\bar{T}_2) = \OOO\left( n\bar{T}_2\right)\\
&= \OOO\left( \frac{2\sqrt{n}\tilde{c}_2^2}{\varepsilon^2} \log \frac{\tilde{c}_2^2}{\varepsilon^2} \right)  = \OOO\left( \frac{2\sqrt{n}\tilde{c}_2^2}{\varepsilon^2} \log \frac{\tilde{c}_2^2}{\varepsilon^2} \wedge \frac{4\tilde{c}_2^3}{3\varepsilon^3}\log \frac{\tilde{c}_2^2}{\varepsilon^2}\right).
\end{align*}

When $\tilde{c}_2 \leq 1.5\sqrt{n}\varepsilon$,
\begin{align*}
&\quad \frac{2\tilde{c}_2^2}{\sqrt{n}\varepsilon^2} \log \frac{\tilde{c}_2^2}{\varepsilon^2} = \frac{4\tilde{c}_2^2}{\sqrt{n}\varepsilon^2} \log \frac{\tilde{c}_2}{\varepsilon} \leq \frac{4\tilde{c}_2^2}{\sqrt{n}\varepsilon^2} \log \frac{3\tilde{c}_2}{\varepsilon}\\
&\leq \frac{1.5\varepsilon}{\tilde{c}_2}\cdot \frac{4\tilde{c}_2^2}{\varepsilon^2} \log \frac{3\tilde{c}_2}{\varepsilon} = \frac{6 \tilde{c}_2}{\varepsilon}\log \frac{3\tilde{c}_2}{\varepsilon}.
\end{align*}
Thus, $\bar{T}_2 = \OOO \left( \frac{ \tilde{c}_2}{\varepsilon}\log \frac{3\tilde{c}_2}{\varepsilon} \right) $. Then
\begin{align*}
&\quad\bar{T}_2^3 \wedge(n\bar{T}_2) = \OOO \left( \bar{T}_2^3\right) = \OOO \left( \frac{ \tilde{c}_2^3}{\varepsilon^3}\left(\log \frac{3\tilde{c}_2}{\varepsilon}\right)^3 \right)\\
&= \OOO \left( \frac{ \tilde{c}_2^3}{\varepsilon^3}\left(\log \frac{3\tilde{c}_2}{\varepsilon}\right)^3 \wedge \frac{1.5\sqrt{n} \tilde{c}_2^2}{\varepsilon^2}\left(\log \frac{3\tilde{c}_2}{\varepsilon}\right)^3 \right).
\end{align*}
To sum up,
$$
{T}_2^3 \wedge(n{T}_2) = \tilde{\OOO}_{L,\Delta_f,\alpha_M}\left(\frac{1}{\varepsilon^3}\wedge \frac{\sqrt{n}}{\varepsilon^2}\right).
$$

\noindent \textbf{$\bullet\ T_3^3\wedge (nT_3)$}

Since $T_3 = \tilde{\Theta}_{L,\Delta_f,\alpha_M}\left( \frac{1}{\varepsilon}\right)$, we can directly know that
$$
{T}_3^3 \wedge(n{T}_3) = \tilde{\OOO}_{L,\Delta_f,\alpha_M}\left(\frac{1}{\varepsilon^3}\wedge \frac{{n}}{\varepsilon}\right) = \tilde{\OOO}_{L,\Delta_f,\alpha_M}\left(\frac{1}{\varepsilon^3}\wedge \frac{\sqrt{n}}{\varepsilon^2}\right).
$$

\noindent \textbf{$\bullet\ T_4^3\wedge (nT_4)$}

Similar to the previous case,
$$
T_4^3\wedge (nT_4) = \tilde{\OOO}_{L,\Delta_f,\alpha_M}\left(\frac{1}{\varepsilon^3}\wedge \frac{\sqrt{n}}{\varepsilon^2}\right).
$$
\end{proof}


\begin{proof}[\textbf{Proof of Theorem \ref{thm:main2_append}}]
We can see that many results given under the setting of Theorem \ref{thm:main1} can still apply under the current setting. If we still define $\Omega_j$ as (\ref{eq:omegaj}), $\Omega = \bigcup\limits_{j=1}\limits^{\infty}\Omega_j$ occurs with probability at least $1-\delta$.

Under the current setting, Proposition \ref{prop:innerloop} is still valid. Thus, on $\Omega$, for any $j\in \ZZ_+$,
\begin{align*}
    &\quad f\left( \bdx_{K_j}\upj\right) -f\left( \bdx_0\upj\right)\\
    &\leq  - \frac{1}{16L}\sum\limits_{k=0}\limits^{K_j-1} \left\| \bdnu_k\upj\right\|^2 +\frac{L^2\eta_j\tau_jl_jK_j^2}{b_j}+\frac{\eta_jK_jq_j}{2}\\
    &=- \frac{1}{16L}\sum\limits_{k=0}\limits^{K_j-1} \left\| \bdnu_k\upj\right\|^2 +\frac{L^2\eta_j\tau_jl_jK_j^2}{b_j}\\
    &=- \frac{1}{16L}\sum\limits_{k=0}\limits^{K_j-1} \left\| \bdnu_k\upj\right\|^2 +\frac{L\tau_jK_j}{4}\\
    &=- \frac{1}{16L}\sum\limits_{k=0}\limits^{K_j-1} \left\| \bdnu_k\upj\right\|^2 + \frac{\sqrt{n}L}{4}\tau_j,
\end{align*}
where the 2nd step is due to our choice of $B_j\equiv n$ and consequently $q_j\equiv 0$, the 3rd step is based on our choices of $\eta_j=\frac{1}{4L}$ and $b_j=l_jK_j$, the 4th step is based on our choice of $K_j=n$. Summing the above inequality from $j=1$ to $T$,
\begin{align}
&\quad -\Delta_f^0\nonumber\\
&= f\left(\bdx^*\right) - f\left(\bdx_0^{(1)}\right)\nonumber\\
&\leq f\left(\bdx_{K_T}^{(T)}\right) - f\left(\bdx_0^{(1)}\right)\nonumber\\
&= \sum\limits_{j=1}\limits^T \left( \frac{\sqrt{n}L}{4}\tau_j - \frac{1}{16L}\sum\limits_{k=0}\limits^{K_j-1} \left\| \bdnu_k\upj\right\|^2 \right)\nonumber\\
&= \frac{\sqrt{n}T\tilde{\varepsilon}^2}{32L} - \frac{1}{16L}\sum\limits_{j=1}\limits^T \sum\limits_{k=0}\limits^{K_j-1} \left\| \bdnu_k\upj\right\|^2,\label{eq:delta0bound}
\end{align}
where the 2nd step is according to Assumption \ref{assump:minimumavailable}, the 4th step is based on our choice of $\tau_j\equiv \frac{\tilde{\varepsilon}^2}{8L^2}$. We assert that when $T\ge T_5$, there must exist a $1\leq j\leq T$ such that
$$
\frac{1}{K_j}\sum\limits_{k=0}\limits^{K_j-1}\left\|\bdnu_k^{(j)}\right\|^2 \leq \tilde{\varepsilon}^2.
$$
If not, 
\begin{align*}
&\quad \frac{\sqrt{n}T\tilde{\varepsilon}^2}{32L} - \frac{1}{16L}\sum\limits_{j=1}\limits^T \sum\limits_{k=0}\limits^{K_j-1} \left\| \bdnu_k\upj\right\|^2\\
&= \frac{\sqrt{n}T\tilde{\varepsilon}^2}{32L}- \frac{1}{16L}\sum\limits_{j=1}\limits^T \frac{K_j}{K_j} \sum\limits_{k=0}\limits^{K_j-1} \left\| \bdnu_k\upj\right\|^2\\
&\leq \frac{\sqrt{n}T\tilde{\varepsilon}^2}{32L}- \frac{1}{16L}\sum\limits_{j=1}\limits^T \tilde{\varepsilon}^2{K_j}\\
&= \frac{\sqrt{n}T\tilde{\varepsilon}^2}{32L} - \frac{\sqrt{n}\tilde{\varepsilon}^2T}{16L}\\
&= -\frac{\sqrt{n}T\tilde{\varepsilon}^2}{32L}\\
&=-\frac{\sqrt{n}\varepsilon^2T}{160L}\\
&\leq -(\Delta_f^0+1),
\end{align*}
which is in conflict with (\ref{eq:delta0bound}). Thus, on $\Omega$, the first stopping rule will be met in at most $T$ outer iterations while the second stopping rule is always satisfied. When both stopping rules are met, we can show that the output is of desirable property. Let $1\leq j\leq T$ and $0\leq k \leq K_j$ such that
$$
\frac{1}{K_j}\sum\limits_{k=0}\limits^{K_j-1}\left\|\bdnu_k^{(j)}\right\|^2 \leq \tilde{\varepsilon}^2
$$
and
$$
\left\|\bdnu_{k}^{(j)}\right\|^2 \leq \tilde{\varepsilon}^2.
$$
Then, on $\Omega$,
\begin{align*}
    &\quad \left\| \nabla f\left( \hat{\bdx}\right) \right\|^2\\
    &= \left\| \nabla f\big( \bdx_{k}\upj\big) \right\|^2\\
    &\leq 2\left\|\bdnu_{k}^{(j)}\right\|^2 + 2\left\|\bdnu_{k}^{(j)}- \nabla f\big( \bdx_{k}\upj\big)\right\|^2\\
    &\leq 2\tilde{\varepsilon}^2 + 2\left\|\bdnu_{k}^{(j)}- \nabla f\big( \bdx_{k}\upj\big)\right\|^2\\
    &\leq  2\tilde{\varepsilon}^2 + 2l_j\left( \frac{4L^2\eta_j^2}{b_j}\sum\limits_{m=1}\limits^k\left\|\bdnu_{m-1}\upj\right\|^2 + \frac{4L^2\tau_j k}{b_j}\right)\\
    &\leq 2\tilde{\varepsilon}^2 + 2l_j\left( \frac{4L^2\eta_j^2}{b_j}\sum\limits_{m=1}\limits^{K_j}\left\|\bdnu_{m-1}\upj\right\|^2 + \frac{4L^2\tau_j K_j}{b_j}\right)\\
    &\leq 2\tilde{\varepsilon}^2 + 2l_j\left( \frac{4L^2\eta_j^2K_j\tilde{\varepsilon}^2}{b_j} + \frac{4L^2\tau_j K_j}{b_j}\right)\\
    &= 2\tilde{\varepsilon}^2 + 0.5\tilde{\varepsilon}^2 +8L^2\tau_j\\
    &=3.5\tilde{\varepsilon}^2\\
    &\leq \varepsilon^2,
\end{align*}
where the 4th step is based on Proposition \ref{prop:nudiff}, the 7th step is based on our choices of $\eta_j=\frac{1}{4L}$ and $b_j=l_jK_j$, the 8th step is based on our choice of $\tau_j\equiv \frac{\tilde{\varepsilon}^2}{8L^2}$.
\end{proof}

\newpage
\section{Supplementary Figures}
\label{append:fig}
\begin{figure*}[htbp!]
     \centering
     \begin{subfigure}[b]{0.24\textwidth}
         \centering
         \includegraphics[width=\textwidth]{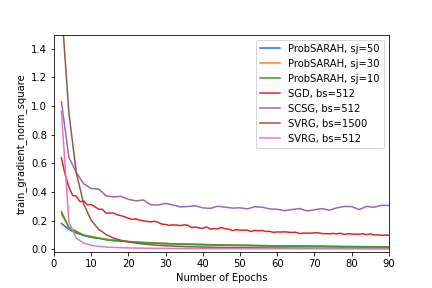}
         \label{fig:r1}
     \end{subfigure}
     \hfill
     \begin{subfigure}[b]{0.24\textwidth}
         \centering
         \includegraphics[width=\textwidth]{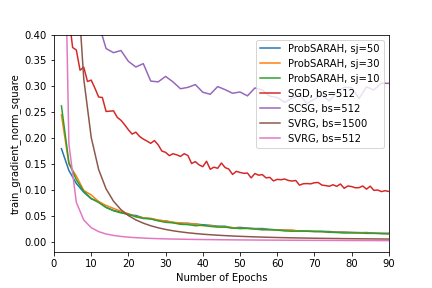}
         \label{fig:r2}
     \end{subfigure}
     \hfill
     \begin{subfigure}[b]{0.24\textwidth}
         \centering
         \includegraphics[width=\textwidth]{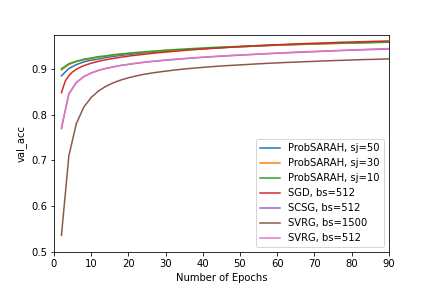}
         \label{fig:r3}
     \end{subfigure}
     \hfill
     \begin{subfigure}[b]{0.24\textwidth}
         \centering
         \includegraphics[width=\textwidth]{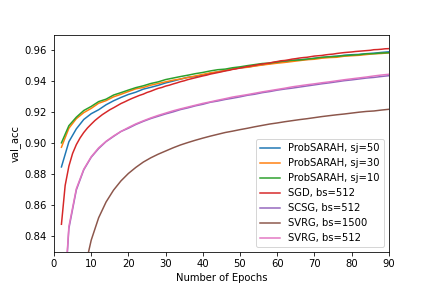}
         \label{fig:r4}
     \end{subfigure}
     
     \begin{subfigure}[b]{0.24\textwidth}
         \centering
         \includegraphics[width=\textwidth]{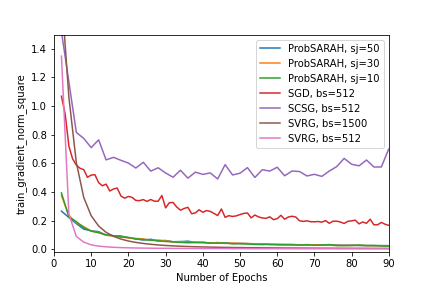}
         \label{fig:r5}
     \end{subfigure}
     \hfill
     \begin{subfigure}[b]{0.24\textwidth}
         \centering
         \includegraphics[width=\textwidth]{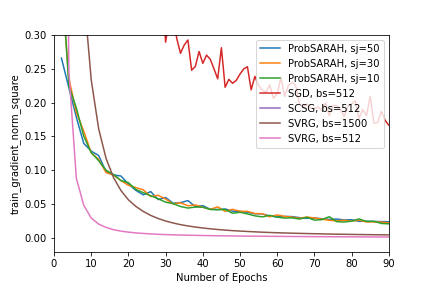}
         \label{fig:r6}
     \end{subfigure}
     \hfill
     \begin{subfigure}[b]{0.24\textwidth}
         \centering
         \includegraphics[width=\textwidth]{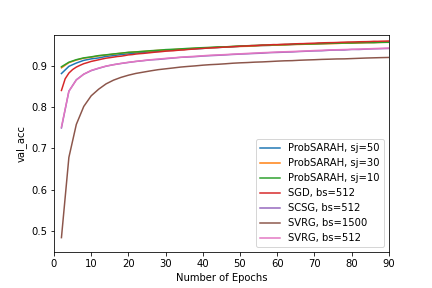}
         \label{fig:r7}
     \end{subfigure}
     \hfill
     \begin{subfigure}[b]{0.24\textwidth}
         \centering
         \includegraphics[width=\textwidth]{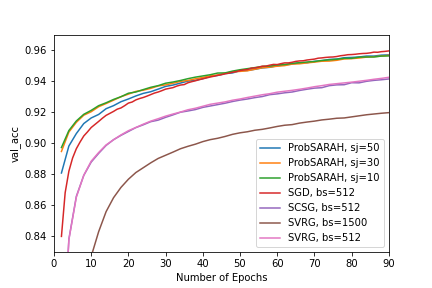}
         \label{fig:r8}
     \end{subfigure}
     \caption{Comparison of convergence with respect to $(1-\delta)$-quantile of square of gradient norm $\left( \|\nabla f\|^2\right)$ and $\delta$-quantile of validation accuracy on the \textbf{MNIST} dataset for $\delta=0.1$ and $\delta=0.01$. The second (fourth) column presents zoom-in figures of those in the first (third) column. Top: $\delta=0.1$. Bottom: $\delta=0.01$. 'bs' stands for batch size. 'sj=x' means that the smallest batch size $\approx x\log x$.}
     \label{fig:mnist2}
\end{figure*}



\end{appendix}

\printbibliography

\end{document}